\documentclass[a4paper,11pt]{article}

\usepackage{amsmath,amssymb,amsfonts,amsthm}

\usepackage[preprint,nonatbib]{neurips_2024}
\usepackage[T1]{fontenc}    
\usepackage{url}            
\usepackage{booktabs}       
\usepackage{amsfonts}       
\usepackage{nicefrac}       
\usepackage{microtype}      
\usepackage{xcolor}         

\usepackage{pgfplots}
\pgfplotsset{compat=1.15} 
\usepackage{pgfplotstable}
\pgfplotsset{compat=1.15}

\usepackage{enumitem}
\usepackage{amssymb, amsmath}
\usepackage{mathrsfs}
\usepackage{amscd} \usepackage[active]{srcltx} \usepackage{verbatim}
\usepackage[colorlinks,linkcolor={blue},citecolor={blue},urlcolor={black}]{hyperref}
\usepackage[utf8]{inputenc}

\usepackage[colorinlistoftodos,prependcaption,textsize=footnotesize, color=red!20]{todonotes}

\usepackage{tikz}
\usetikzlibrary{backgrounds}
\usetikzlibrary{decorations.pathmorphing}
\usetikzlibrary{arrows}
\usetikzlibrary{shapes.geometric}
\usetikzlibrary{decorations.pathreplacing,calc}
\usetikzlibrary{calc}
\usetikzlibrary {arrows.meta,bending,positioning}

\usepackage{tikz}
\usepackage{lmodern}

\usepackage{xcolor}
\usepackage{caption}
\usepackage{wrapfig}

\theoremstyle{plain}
\newtheorem{theorem}{Theorem}
\newtheorem{lemma}{Lemma}[section]

\newtheorem{proposition}[lemma]{Proposition}

\newtheorem{definition}[lemma]{Definition}
\newtheorem{assumption}[lemma]{Assumption}

\newtheorem*{definition*}{Definition}

\theoremstyle{remark}
\newtheorem{remark}[lemma]{Remark}

\newtheorem*{claim*}{Claim}
\newtheorem*{remark*}{Remark}
\newtheorem*{example*}{Example}

\newtheorem*{notation*}{Notation}

\numberwithin{equation}{section}

\newcommand{\x}{{\mathbf x}}

\renewcommand{\phi}{\varphi}

\newcommand{\PP}{\mathbb{P}}

\DeclareMathOperator{\argmin}{argmin}

\DeclareMathOperator{\sgn}{sgn}

\newcommand{\TK}[1]{\tau^{\rm Ker}_{#1}}

\newcommand{\TPM}{\tau^{\rm Pred}_M}

\newcommand{\cA}{\mathcal{A}}

\newcommand{\RRd}{\mathbb{R}^d}

\newcommand{\RR}{\mathbb{R}}

\newcommand{\cF}{\mathcal{F}}

\renewcommand{\tilde}{\widetilde}

\newcommand{\hs}{\ensuremath{\hspace{1cm}}}
\newcommand{\h}{\ensuremath{\hspace{0.1cm}}}
\newcommand{\EE}{\ensuremath{\mathbb{E}}}

\newcommand{\W}{\ensuremath{\mathbf{W}}}
\newcommand{\p}{\ensuremath{\mathbf{p}}}

\newcommand{\NN}{\ensuremath{\mathbb{N}}}

\newcommand{\indiq}{\hbox{\rm 1}{\hskip -2.8 pt}\hbox{\rm I}}

\def\highlights{0} 
 \ifx\highlights\undefined
\newcommand{\edit}[1]{#1}
\else
  \if\highlights1
\newcommand{\edit}[1]{{\color{red}{#1}}}
  \else
\newcommand{\edit}[1]{#1}
  \fi
\fi

\newcommand\QQ{\mathbb{Q}}

\newcommand{\step}[1]{\paragraph{\textbf{Step #1}}}

\newcommand{\buildingsite}[1]{}
\newcommand{\findme}{}

\usepackage{fancyhdr}
\makeatletter

\title{Large Spikes in Stochastic Gradient Descent: A Large-Deviations View}

\author{
Benjamin Gess \\
Technische Universit\"at Berlin, Fakult\"at II \\
10623 Berlin, Germany \\
Max Planck Institute for Mathematics in the Sciences \\
04103 Leipzig, Germany \\
\texttt{benjamin.gess@mis.mpg.de}
\And
Daniel Heydecker \\ University of Oslo, Blindern\\ 0316 Oslo, Norway \\ 
\texttt{daniehey@math.uio.no}
}

\begin{document}
\maketitle

\begin{abstract}
Large loss spikes in stochastic gradient descent are studied through a rigorous large-deviations analysis for a shallow, fully connected network in the NTK scaling. In contrast to full-batch gradient descent, the catapult phase is shown to split into inflationary and deflationary regimes, determined by an explicit log-drift criterion. In both cases, large spikes are shown to be at least polynomially likely. In addition, these spikes are shown to be the dominant mechanism by which sharp minima are escaped and curvature is reduced, thereby favouring flatter solutions. Corresponding results are also obtained for certain ReLU networks, and implications for curriculum learning are derived.
\end{abstract}

\makeatother
\lhead{Spikes in SGD: an LDP View}
\rhead{Gess, Heydecker}



\parindent=0cm



\def\newintro{2} 

 \section{Introduction}


\ifx\newintro\undefined
undefd
\else 
\if\newintro2

    Modern machine learning problems often involve training neural networks which may have millions, billions, or even trillions of parameters \cite{fedus2022switch}. Remarkably, this ubiquitous over-parametrisation does not appear to produce overfitting, and first order training methods are empirically observed to produce networks which generalise well to unseen data \cite{zhang2021understanding}. Consequently, understanding why and how training methods select well-generalising minima - and how hyperparameters or training regimes may be chosen to ensure good generalisation of the solution - remains an important open problem:  {\em Given a loss function with many minimisers, which are favoured by full-batch and stochastic gradient descent algorithms?} \cite{pesme2021implicit, vasudeva2025rich,blanc2020implicit,li2021happens,shalova2024singular,marion2024deep,andreyev2024edge,arora2022understanding,damian2022self,wu2022alignment}.
    
    A cornerstone of modern optimisation is {\em stochastic gradient descent} (SGD) via mini-batching \cite{robbins1951stochastic}, as detailed in \eqref{eq: SGD}. In regimes of small learning rates or asymptotic batch sizes, SGD is effectively approximated by gradient descent (GD). One way in which gradient-based methods with discrete updates can escape from sharp minima and decrease sharpness is the {\em catapult mechanism}, where an individual update causes the updated prediction to overshoot and increase the loss \cite{lewkowycz2020large,zhu2022quadratic}, causing the dynamics to spike and ultimately settle in flatter minima. Such curvature was identified in \cite{hochreiter1997flat} as a relevant criterion for generalisation; in the deterministic setting, this catapult typically manifests as a singular, pronounced spike in the loss.
  
    Beyond this, remarkably, it has been found that the stochasticity inherent to SGD may produce convergence to {\em better} generalising minima than those found by the deterministic gradient descent \cite{frankle2020early,smith2019super,gilmer2021loss}.  This effect is particularly strong if the learning rate $\eta$ is large or the batch size $b$ is small \cite{xie2020diffusion,jastrzkebski2017three,li2019towards,leclerc2020two,lu2023benign}, that is, in settings where the GD approximation of SGD fails, and fluctuations cannot be neglected. \\
    In addition, distinctly from the singular events observed in GD, SGD is frequently punctuated by multiple transient, high-amplitude spikes in the loss $\ell(\Theta(t))$ which are well-documented in empirical practice \cite{lecun2015deep, ruder2016overview, keskar2017improving, xing2018walk} and are associated with a tendency to improve generalisation \cite{he2016deep, zagoruyko2016wide, huang2017densely}. Recent empirical evidence \cite{zhu2023catapults} suggests that this improvement in generalisation, as well as the frequent spike dynamics, may be caused by an interaction of stochasticity and catapult dynamics, leading to the following central challenge:  
  \begin{equation*}
  \begin{split} 
  &\text{\textit{How does the interplay between the catapult mechanism and stochastic noise help escape}} \\ 
  &\hspace{5cm} \text{\textit{from sharp minima?}} 
  \end{split} 
  \end{equation*} 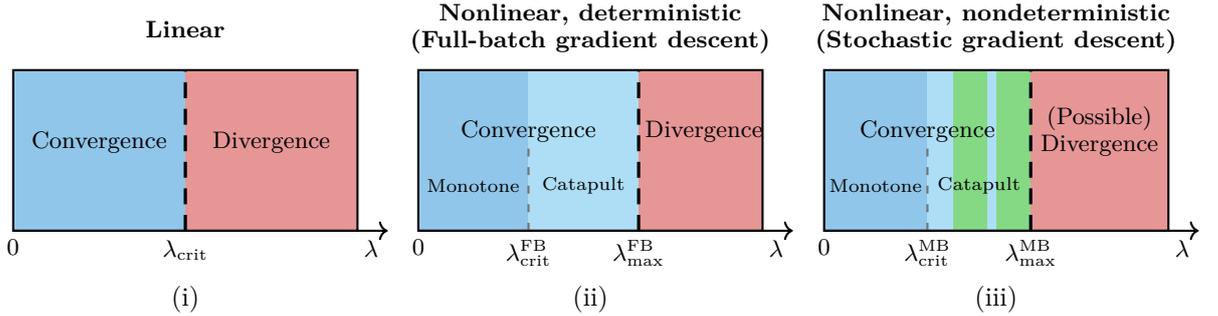
\begin{figure}[t]
\centering 
\begin{tikzpicture}[font=\footnotesize]

\definecolor{convblue}{RGB}{144,198,233}   
\definecolor{convblue2}{RGB}{173,222,245}  
\definecolor{divred}{RGB}{231,150,150}     
\definecolor{stripegreen}{RGB}{33, 186, 34}

\tikzset{
  ax/.style={->, line width=0.8pt},
  dline/.style={dash pattern=on 6pt off 4pt, line width=1.2pt},
  regionborder/.style={line width=0.8pt},
  title/.style={font=\bfseries\small},
  biglabel/.style={font=\bfseries\Large},
  midlabel/.style={font=\small},
  smalllabel/.style={font=\scriptsize},
  tinylabel/.style={font=\tiny}
}

\def\W{4.5}  
\def\H{2.13}  
\def\gap{0.8}
\def\s{10.67}


\begin{scope}[shift={(0,0)}]
  \def\xcrit{0.40*\W}

  \fill[convblue] (0,0) rectangle (1.25*\xcrit,\H);
  \fill[divred]   (1.25*\xcrit,0) rectangle (\W,\H);

  \draw[regionborder] (0,0) rectangle (\W,\H);
  \draw[dline] (1.25*\xcrit,0) -- (1.25*\xcrit,\H);

  \draw[ax] (0,0) -- (\W+0.4,0);

  \node[title] at (0.5*\W,\H+0.55) {Linear};
  \node[midlabel] at (0.25*\W,0.55*\H) {Convergence};
  \node[midlabel] at (0.75*\W,0.55*\H) {Divergence};

  \node[anchor=north] at (0,0) {$0$};
  \node[anchor=north] at (1.25*\xcrit,0) {$\lambda_{\mathrm{crit}}$};
  \node[anchor=north east] at (\W+0.4,0) {$\lambda$};

  \node[font=\normalsize] at (0.5*\W,-0.9) {(i)};
\end{scope} 
\begin{scope}[shift={(\W+\gap,0)}]
  \def\xcrit{0.32*\W}
  \def\xmax{0.64*\W}

  \fill[convblue]  (0,0) rectangle (\xcrit,\H);
  \fill[convblue2] (\xcrit,0) rectangle (\xmax,\H);
  \fill[divred]    (\xmax,0) rectangle (\W,\H);

  \draw[regionborder] (0,0) rectangle (\W,\H);
  \draw[dline] (\xmax,0) -- (\xmax,\H);
  \draw[dline, line width=0.8pt, dash pattern=on 3pt off 4pt, black!55] (\xcrit,0) -- (\xcrit,0.55*\H);

  \draw[ax] (0,0) -- (\W+0.4,0);

  \node[title, align=center] at (0.5*\W,\H+0.55)
  {Nonlinear, deterministic\\(Full-batch gradient descent)};  \node[midlabel] at (0.32*\W,0.62*\H) {Convergence};
  \node[smalllabel] at (0.16*\W,0.28*\H) {Monotone};
  \node[smalllabel, align=center] at (0.48*\W,0.28*\H) {Catapult};
  \node[midlabel] at (0.83*\W,0.62*\H) {Divergence};

  \node[anchor=north] at (0,0) {$0$};
  \node[anchor=north] at (\xcrit,0) {$\lambda^{\rm FB}_{\mathrm{crit}}$};
  \node[anchor=north] at (\xmax,0) {$\lambda^{\rm FB}_{\max}$};
  \node[anchor=north east] at (\W+0.4,0) {$\lambda$};

  \node[font=\normalsize] at (0.5*\W,-0.9) {(ii)};
\end{scope}


\begin{scope}[shift={(0.5*\W+0.5*\gap,-2.2*\H)}]
  \def\xcrit{0.30*\W}
  \def\xmax{0.60*\W}

  \fill[convblue] (0,0) rectangle (\xcrit,\H);

 \pgfmathsetmacro{\stripeW}{(\xmax-\xcrit)/12}

 \fill[convblue2] (\xcrit,0) rectangle (\xcrit+3*\stripeW,\H);

 \fill[stripegreen!55] (\xcrit+3*\stripeW,0) rectangle (\xcrit+7*\stripeW,\H); 
    
  \fill[convblue2] (\xcrit+7*\stripeW,0) rectangle (\xcrit+8*\stripeW,\H);
      
 \fill[stripegreen!55] (\xcrit+8*\stripeW,0) rectangle (\xcrit+12*\stripeW,\H); 

  \fill[divred] (\xmax,0) rectangle (\W,\H);

  \draw[regionborder] (0,0) rectangle (\W,\H);
  \draw[dline] (\xmax,0) -- (\xmax,\H);
  \draw[dline, line width=0.8pt, dash pattern=on 3pt off 4pt, black!55] (\xcrit,0) -- (\xcrit,0.55*\H);

  \draw[ax] (0,0) -- (\W+0.4,0);

  \node[title, align=center] at (0.5*\W,\H+0.55)
  {Nonlinear, nondeterministic\\(Stochastic gradient descent)};
  \node[midlabel] at (0.30*\W,0.62*\H) {Convergence};
  \node[smalllabel] at (0.15*\W,0.28*\H) {Monotone};
  \node[smalllabel, align=center] at (0.45*\W,0.28*\H) {Catapult};
  \node[midlabel, align=center] at (0.8*\W,0.62*\H)
  {(Possible)\\Divergence};
 
  \node[anchor=north] at (0,0) {$0$};
  \node[anchor=north] at (\xcrit,0) {$\lambda^{\rm MB}_{\mathrm{crit}}$};
  \node[anchor=north] at (\xmax,0) {$\lambda^{\rm MB}_{\max}$};
  \node[anchor=north east] at (\W+0.4,0) {$\lambda$};

  \node[font=\normalsize] at (0.5*\W,-0.9) {(iii)};
\end{scope}
\end{tikzpicture}
\caption{Extension of the phase diagram \cite[Figure 1]{zhu2022quadratic}, which corresponds to cases (i-ii). Theorem \ref{thrm: informal} unveils a rich internal structure of the catapult region for the nonlinear, nondeterministic dynamics \eqref{eq: SGD}, illustrated with green (`inflationary', case a of Theorem \ref{thrm: informal}) and pale blue (`deflationary', case b) stripes in (iii). In general, the critical and maximal curvatures $\lambda^{\rm MB}_{\rm crit, max}$ for mini-batching are strictly smaller than their full-batch counterparts $\lambda^{\rm FB}_{\rm crit, max}$.} \label{fig: phases intro}  
\end{figure}  
  The objective of this work is to establish a rigorous mathematical framework to elucidate these phenomena {in the context of a model example}. Specifically, we demonstrate that the phase diagram of catapults fundamentally changes in comparison to the deterministic GD case, 
  yielding a diverse range of possible configurations, see Figures \ref{fig: phases intro}-\ref{fig:GPlot}. In addition to changing the concrete values of the \textit{phase transition}, the catapult phase exhibits non-trivial internal structure. In fact, we exhibit two distinct behaviours within the catapult phase, depending on the sign of an explicit function $G(\lambda, \eta)$, each associated to distinct spiking behaviour. Both cases explain the occurrence of multiple spikes as observed in practice, and a corresponding stronger flattening effect of the dynamics.

   We further emphasise that the (stochastic) dynamical effects are different from those that have been observed in other works on large deviations in SGD. Indeed, the dynamics in the present work are driven by the catapult mechanism, a feature specific to the discrete-time setting, while previous works on the escape from minima by large deviations work in a regime of infinitesimally small learning rate, in which no catapults appear \cite{ibayashi2021quasi,azizian2024long,azizian2025global}.

   Finally, the analysis of the implicit bias dynamics developed in this work leads to principled curriculum learning strategies which may help enhance training, see Section \ref{sec: curriculum learning} below.

\else
  \if\newintro1
  Modern machine learning problems often involve training neural networks which may have millions, billions, or even trillions of parameters \cite{fedus2022switch}. Remarkably, this ubiquitous over-parametrisation does not appear to produce overfitting, and first order training methods are empirically observed to produce networks which generalise well to unseen data \cite{zhang2021understanding}. Consequently, understanding why and how training methods select well-generalising minima - and how hyperparameters or training regimes may be chosen to ensure good generalisation of the solution - remains an important open problem.

Among the most important and widely used training methods for minimising a linearly separable, but often non-convex loss $\ell$, is {\em stochastic gradient descent (SGD) by mini-batching} \cite{robbins1951stochastic}, recalled at \eqref{eq: SGD} below. Remarkably, it appears that the stochasticity produces convergence to {\em better} generalising minima than those found by the deterministic gradient descent \cite{frankle2020early,smith2019super,gilmer2021loss}. This raises the question of {implicit bias}: {\em Given a loss function with many minimisers, which do the full-batch and stochastic gradient descent algorithms tend to prefer?} \cite{pesme2021implicit, vasudeva2025rich,blanc2020implicit,li2021happens,shalova2024singular,marion2024deep,andreyev2024edge,arora2022understanding,damian2022self,wu2022alignment}. A proposed explanation for the better generalisation properties enjoyed by SGD is that the randomness introduced by mini-batching helps the trajectory escape sharp minima, as measured by the loss Hessian \cite{kleinberg2018alternative,zhou2020towards,keskar2016large,jastrzkebski2018finding,xie2020diffusion,azizian2024long}, and settle in flatter minima, where such curvature was identified in \cite{hochreiter1997flat} as a relevant criterion for generalisation.  This effect is particularly strong if the learning rate $\eta$ is large or the batch size $b$ is small \cite{xie2020diffusion,jastrzkebski2017three,li2019towards,leclerc2020two,lu2023benign}. \\\\ One way in which gradient-based methods with discrete updates can escape from sharp minima is the catapult mechanism, where an individual update causes the updated prediction to overshoot and increase the loss. This mechanism has also been proposed empirically \cite{zhu2023catapults} as an explanation for the many large, short-lived spikes in the loss $\ell(\Theta(t))$ which are well-known to practitioners \cite{lecun2015deep,ruder2016overview,keskar2017improving,xing2018walk} to punctuate the convergence of SGD, as well as the observed tendency of such spikes to improve generalisation \cite{he2016deep,zagoruyko2016wide,huang2017densely}. \\\\ This explanation, however, raises the question of what benefit - if any - towards generalisation is gained by mini-batching, as the same mechanism was shown in the deterministic setting \cite{lewkowycz2020large,zhu2022quadratic} to contribute towards finding flatter minima. Precisely, there is a phase transition where, for a range of sufficiently large learning rates, the catapult mechanism first produces a large spike, followed by a rapid curvature reduction and convergence to a minimum with smaller (scalar) curvature. The task of this work is thus to develop a mathematically precise theory in order to understand \begin{equation*}\begin{split} &\text{\em How does the interplay between the catapult mechanism and SGD noise help escape}\\ &\hspace{5cm} \text{\em from sharp minima?}  \end{split} \end{equation*}
  \else
  
Modern machine learning processes rely on various optimisation techniques to find solutions of optimisation problems \begin{equation}
	\label{eq: opt} \Theta^\star\in \argmin_{\Theta \in \RR^n} \ell(\Theta)
\end{equation} where $\ell$ is a typically non-convex loss, and the number of parameters $n$ is extremely large: At the time of writing, neural networks may contain millions, billions or trillions of parameters. Typically, to learn the parameters of a network, the loss will be of the form \begin{equation}\label{eq: prototype loss} 
  \ell(\Theta)=\frac1{m}\sum_{i=1}^{m}\ell_{i}(\Theta); \qquad \ell_i(\Theta)=L(F(\Theta, x_i), y_i)
\end{equation}
for some prediction map $F:\RR^n\times \RRd \to \RR$, and a given set of labelled training data $\{(x_i, y_i)\}$. For such losses, a simple, computationally tractable and widely used approach \cite{robbins1951stochastic}, profiting from the form of the total loss $\ell$, is {\em stochastic gradient descent (SGD) by minibatching}: given the parameters $\Theta(t)$ at a discrete time step $t\in \NN$, a subset $B(t+1)\subset [m]$ of fixed size $b\ll m$ is chosen at random, and the parameters are updated by the discrete scheme  \begin{equation} \tag{SGD} \label{eq: SGD}
  \Theta({t+1})=\Theta({t})-\frac{\eta}{b}\sum_{i\in B(t+1)}\nabla_{\Theta}\ell_{i}(\Theta({t}))
\end{equation} where $\eta>0$ is a fixed step size.\\ \\ Remarkably, it appears that the stochastic nature of \eqref{eq: SGD} produces convergence which is {\em better} than that of the deterministic gradient descent: It has been empirically observed \cite{keskar2016large} that \eqref{eq: SGD}, particularly with small batch size $b$, tends to select flatter, and better-generalising minima, particularly when the learning rate $\eta$ is large or the batch size $b$ is small \cite{li2019towards,leclerc2020two,lu2023benign}. Meanwhile, it is well-known to practitioners that the convergence of SGD is typically punctuated by many large, short-lived spikes in the loss $\ell(\Theta(t))$, each of which is followed by a rapid return to its pre-spike value, attributed empirically by \cite{zhu2023catapults} to the {\em catapult mechanism}. It has also been argued \cite{lewkowycz2020large,zhu2022quadratic} that the catapult mechanism, even in the full-batch case, may drive short-lived, finite-width effects driving the system to better-behaved minima. The task of this work is to develop a mathematically precise understanding how the structure of SGD noise interacts with the catapult mechanism, the structure of the resulting phase transition, and contribution of spikes to finding flatter minima.
	 \fi
  \fi
\fi


\subsection{Contribution}

  The contribution of this work is a rigorous understanding of the structure of the catapult phase in SGD and its contribution to finding flatter minima. We use the neural tangent kernel (NTK), which in our examples reduces to a scalar $\lambda$ or a pair of scalars, as a measure of the sharpness. We introduce and rigorously analyse a model example, following \cite{zhu2022quadratic,meltzer2025catapult}, for shallow, fully connected networks in the NTK scaling, in the sharp minimum range $\eta\lambda\sim 1$. In this setting, we are able to completely categorise when spikes are guaranteed, and how likely they are when not guaranteed. Even for such a simplified model, the catapult phase shows rich and non-trivial internal structure which is not present in the case of full-batch gradient descent. The structure is illustrated in Figure \ref{fig: phases intro}, stated precisely in Theorem \ref{thrm: informal} and exemplified in Section \ref{sec: examples}.\\ \\ We summarise as follows. Defining a {\em large spike} to be an event where the loss $\ell(\Theta(t))$ reaches the threshold ${n/\eta}$ where the network escapes the linear regime of lazy training and where the curvature is reduced by $\mathcal{O}(1)$, the results may be informally summarised as follows. The contribution is that: \begin{enumerate} \item For a certain phase of initial curvature $\lambda_0$, a large spike is guaranteed (Theorems \ref{thm:lln moderate} and \ref{thrm: LLN large});  \item In a further regime, in which they are not guaranteed, they remain {\em polynomially likely}, in that their probability decays like $(n/\eta)^{-\vartheta/2}$, for some $\vartheta=\vartheta(\lambda_0) \in (0,\infty)$ (Theorem \ref{thm: LDP spikes} and \ref{thrm: LDP large spikes}); \item Large spikes are, up to exponentially unlikely events, the only way in which the lazy training regime can be escaped and the curvature reduced (Proposition \ref{prop: slow escape}); \item In the case where the activation function of the first layer is the rectified linear unit (ReLU) and a certain asymmetric initialisation is imposed, the dynamics decouple into two copies of the linear model, to which each of the previous points applies. \item As a practical consequence of 1-4, in Section \ref{sec: curriculum learning}, the results produce two ways in which curriculum learning may be applied in order to make benign, curvature-reducing spikes more likely.\end{enumerate} Here, we follow the convention \cite{zhu2023catapults} that a `catapult' refers to a single update step in which the learning rate $\eta$ exceeds a per-sample critical rate, causing an instantaneous increase in the running loss, and distinguish between this and a {\em spike} where the overall loss has become large. \\ \\ Both points 1-2 allow direct computation from the data. The distinction between the two regimes where spikes are typical or atypical is governed by the sign of a function $G(\lambda_0)$, which is given explicitly from the data by \eqref{eq: introduce G}. The function $\vartheta(\lambda_0)$ in point 2 is characterised as the unique positive zero of a certain convex function \eqref{eq: def vartheta informal}, and hence may be found in practice from the data by interval bisection.\\\\ The points 2-3 above arise out of the machinery of large deviations. Often, the theory of large deviations leads to estimates of probabilities decaying exponentially $\sim e^{-\alpha n}, \alpha>0$, and for the values of $10^6 \le n\le 10^{12}$ used in practice, such events may be excluded from consideration. In contrast, the polynomial decay of probabilities discussed in (2) means that such events may retain appreciable probability in practice when the exponent $\vartheta$ is not too large. An example, where the probability remains on the order $(n/\eta)^{-\vartheta/2}\approx 0.25$ for $n=10^{12}$, is given in Section \ref{ex: non monotone behaviour}. In (3), we follow the maxim of large deviations in describing the `least unlikely' way in which random fluctuations via the mini-batching noise can reduce the curvature, and we will see that large kernel reductions without a spike are exponentially unlikely.  

  \subsection{The Model \& Main Results} \label{sec: linear intro} As is common in the literature \cite{williams2019gradient,savarese2019infinite} for developing a theoretical understanding, we consider a univariate network, and specialise further to the case of mini-batch size $b=1$. We consider learning the network $F:\RR^{2n}\times \mathbb{R}\to \mathbb{R}$, given by \begin{equation}\label{eq: toy model}
		F(\Theta; s):=\frac1{\sqrt{n}}\sum_{r=1}^n a_r \varphi(w_r s); \qquad \Theta=((w_r, a_r): 1\le r\le n)\in \RR^{2n}
	\end{equation} where $\varphi$ is either the linear activation $\varphi(w)=w$, or the rectified linear unit (ReLU) $\varphi(w)=\max(w,0)$, and with a quadratic loss function $L(F(\Theta, s),y)=\frac12|F(\Theta, s)-y|^2$.  For a training dataset, we consider a set $\{s_i, 1\le i\le m\}$, all equipped with label $y_i=0$, and sampled with potentially different probabilities $p_i$. For linear activation, the prediction and neural tangent kernel (NTK) are the scalar quantities \begin{equation} \label{eq: prediction linear} \mu(\Theta):=\frac1{\sqrt{n}}\sum_{r=1}^n a_r w_r = F(\Theta, 1); \qquad \lambda(\Theta):=\frac1n\sum_{r=1}^n(w_r^2+a_r^2) \end{equation} while for ReLU, there are pairs of such quantities \begin{equation} \label{eq: prediction RELU} \mu^{\pm}(\Theta):=\frac1{\sqrt{n}}\sum_{r=1}^n a_r \sigma(\pm w_r) = F(\Theta, \pm 1); \qquad \lambda^{\pm}(\Theta):=\frac1n\sum_{r=1}^n(w_r^2+a_r^2)\indiq(w_r>0). \end{equation}  

	We consider parameters $\Theta=\Theta(t)$ updated in discrete time steps $t=0,1,\dots$ as follows. Given the parameters $\Theta(t)$ at a discrete time step $t\in \NN$, a subset $B(t+1)\subset [m]$ of fixed size $b\ll m$ is chosen at random, and the parameters are updated by the discrete scheme  \begin{equation} \tag{SGD} \label{eq: SGD}
  \Theta({t+1})=\Theta({t})-\frac{\eta}{b}\sum_{i\in B(t+1)}\nabla_{\Theta}L(F(\Theta({t}),s_i),y_i)
\end{equation} where $\eta>0$ is a fixed learning rate. In the sequel, we consider only mini-batch size $b=1$, with $B(t+1)=\{i(t+1)\}$ sampled independently across time steps according to the distribution $\{p_i\}$, and write $\mu(t), \lambda(t)$ for the prediction and curvature $\mu(t)=\mu(\Theta(t)), \lambda(t)=\lambda(\Theta(t))$ given by the formul{\ae} \eqref{eq: prediction linear}. We make the following standing assumption, which ensures that the random draws $i(t)$ are the only source of randomness:  \begin{assumption}\label{hyp: assumption linear model} \begin{enumerate} \item For the model with linear activations, let $\mu_0\in \mathbb{R}$, $\lambda_0>0$, and write $\PP_{\mu_0, \lambda_0}$ for a probability measure under which the indexes $i(t)$ are independent and identically distributed draws from the distribution $\{p_i: 1\le i\le m\}$, the parameters $\Theta(t)=((w_r(t), a_r(t)):1 \le r\le n)$ are updated according  to \eqref{eq: SGD} and satisfy, almost surely, $\mu(\Theta(0))=\mu_0, \lambda(\Theta(0))=\lambda_0$. \item For the case with ReLU activations, we write similarly $\PP_{\mu_0^{\pm}, \lambda_0^\pm}$ for a probability measure where $\{i(t)\}, \Theta(t)$ are as above, and $\mu^\pm(\Theta(0))=\mu^{\pm}_0, \lambda^{\pm}(\Theta(0))=\lambda^{\pm}_0$ almost surely. \end{enumerate}	\end{assumption}  With these fixed, we summarise the results, stated precisely in Theorems \ref{thm:lln moderate}, \ref{thm: LDP spikes}, \ref{thrm: LLN large}, \ref{thrm: LDP large spikes} as follows. \begin{theorem}
			\label{thrm: informal} Let $\varphi$ be the linear activation $\varphi(w)=w$, and consider the range \begin{equation}
				\label{eq: minibatch critical values} \lambda^{\rm MB}_{\rm crit}:=\frac2{\eta s_\star^2} < \lambda_0 < \frac{4}{\eta s_\star^2}=:\lambda^{\rm MB}_{\rm max}; \qquad s_\star:=\max_i |s_i|.
			\end{equation}  Define, for any $\lambda$ in this range, \begin{equation} \label{eq: introduce G} G(\lambda)=G(\lambda; \eta, \{s_i, p_i\}):=\sum_{i=1}^m p_i\log\left(\left|1-\eta \lambda s_i^2\right|\right)\in [-\infty, \infty)\end{equation} which we allow to take the value $-\infty$ if $\eta\lambda s_i^2=1$ for some $i$, and \begin{equation}
				\label{eq: def vartheta informal} \vartheta(\lambda)=\sup\left\{\theta\ge 0: \sum_{i=1}^mp_i|1-\eta \lambda s_i^2|^\theta\le 1\right\}. 
			\end{equation} Then \begin{enumerate}[label=\alph*)]
				\item {\em [Inflationary Case]} If $G(\lambda_0)>0$, then, with high probability over the dynamics, the loss $|\mu(t)|^2$ reaches any threshold $$L\lesssim \frac{n}{\eta \log^2(n/\eta)}$$ \findme in time $\mathcal{O}(\frac{\log(L/|\mu_0|^2)}{G(\lambda_0)})$. The spike then produces a reduction in the kernel $\lambda$ reducing it to at most some explicit $\lambda_\star\in (0, \lambda_0)$ in a duration of the order $\log((\lambda_0-\lambda_\star)\log(n/\eta))+\mathcal{O}_{\mathbb{P}}(1)$. \item {\em [Deflationary Case]} If instead $\lambda^{\rm MB}_{\rm crit}<\lambda_0<\lambda^{\rm MB}_{\rm max}$ and $G(\lambda_0)<0$, then the probability of the loss reaching any threshold $$|\mu_0|^2\le L \lesssim \frac{n}{\eta \log^{2/\beta}(n/\eta)}$$ \findme decays like $$\sim \left(\frac{|\mu_0|^2}{L}\right)^{\vartheta(\lambda_0)/2+\mathfrak{o}(1)},$$ where $\vartheta(\lambda_0)>0$ is given by \eqref{eq: def vartheta informal}, and $\beta>0$ is fixed. If, for $\delta>0$, $p_{\delta}:=\PP(|1-\eta\lambda_0s_i^2|>1+\delta)>0$, then the probability of a large spike decreasing the kernel to $\lambda$ in the range $\lambda_0-C\lambda^{\rm MB}_{\rm max}\delta \le \lambda \le \lambda_0$ decays no faster than \begin{equation}
					\label{eq: large spike informal} \gtrsim \left(\frac{|\mu_0|}{\sqrt{n/\eta}}\right)^{\vartheta(\lambda_0)+\mathfrak{o}(1)}(\lambda_0-\lambda)^{\alpha}
				\end{equation} where $\alpha=\alpha(\delta)\in (0,\infty)$ may be found explicitly from $\delta, p_{3\delta}$.
			\end{enumerate} \end{theorem} For the case where $\varphi$ is the ReLU activation $\varphi(w)=\sigma(w)=\max(0,w)$, we first introduce some further notation. We define the positive and negative parts of each data point $s_i^{\pm}=\max(\pm s_i, 0)$, and corresponding critical values $\lambda^{\rm MB, \pm}_{\rm crit}, \lambda^{\rm MB, \pm}_{\rm max}$. Define also $G^\pm(\lambda), \vartheta^{\pm}(\lambda)$ as in (\ref{eq: minibatch critical values}-\ref{eq: def vartheta informal}), with $s_i$ replaced by $s^{\pm}_i$ respectively. With these fixed, we summarise the result for ReLU activation, presented in Theorem \ref{thrm: RELU}, as follows. \begin{theorem} \label{thrm: informal2} Let $\varphi$ be the ReLU activation $\varphi(w)=\sigma(w)=\max(0,w)$, and consider the range \begin{equation} \label{eq: ??} \lambda_0^{\mathfrak{s}}<\lambda^{\rm MB, \mathfrak{s}}_{\rm max}, \qquad \mathfrak{s}\in \{\pm\}.\end{equation} If a certain asymmetric initialisation property in Definition \ref{def: w biased} holds  $\PP_{\mu_0^{\pm}, \lambda_0^{\pm}}$-almost surely, similar to \cite{boursier2024simplicity,boursier2025early,dana2025convergence}, then \begin{enumerate}[label=\alph*)]
				\item {\em [Inflationary Case]} If at least one of $G^\mathfrak{s}(\lambda^\mathfrak{s}), \mathfrak{s}\in \{\pm\}$ is non-negative then, with high probability over the dynamics, the total loss reaches any threshold $$ L\lesssim \frac{n}{\eta \log (n/\eta)}$$ \findme in time $$\mathcal{O}\left(\min\left(\frac{\log(L/|\mu^\mathfrak{s}|^2)}{G^\mathfrak{s}(\lambda^\mathfrak{s}_0)}:  \mathfrak{s}\in \{\pm\}, G^\mathfrak{s}(\lambda^\mathfrak{s}_0)>0\right)\right). $$ \item {\em [Deflationary Case]} If, instead, both $G^{\mathfrak{s}}(\lambda^\mathfrak{s}_0)<0$ and at least one $\lambda^\mathfrak{s}_0>\lambda^{\rm MB, \mathfrak{s}}_{\rm crit}$, then the probability of the total loss reaching any threshold $$L\lesssim \frac{n}{\eta \log^{2/\beta} (n/\eta)}$$ for some fixed $\beta>0$  decays like $$\max\left(\left(\frac{|\mu^\mathfrak{s}_0|^2}{L}\right)^{\vartheta^\mathfrak{s}(\lambda^\mathfrak{s}_0)/2+\mathfrak{o}(1)}: \mathfrak{s}\in \{\pm\}\right).$$ 			\end{enumerate} In this case, it is possible for a large spike to cause a decrease of $\lambda^\mathfrak{s}(t)$ of order $\mathcal{O}(\eta^{-1})$, accompanied by a simultaneous increase of $\lambda^\mathfrak{-s}(t)$. 
		\end{theorem} \begin{remark} Assumption \ref{hyp: assumption linear model} allows us to isolate the randomness associated with the noise sampling, which is responsible for the structure of the catapult phase. Other sources of randomness may be included as follows.
		\begin{enumerate} \item {\bf Random initialisation.} In the case where the initialisation $\Theta(0)$ is itself random, for example with LeCun initialisation \cite{lecun2002efficient}, the initial values of $\mu, \lambda$ may also be random. In this case, $\PP_{\mu_0, \lambda_0}$ may be taken as regular conditional probability measures, conditional on these values, and the conclusions of Theorems \ref{thrm: informal} - \ref{thrm: informal2} may be extended to holding with high probability in the initial data by arguing as in \cite[Theorems 1-2]{zhu2022quadratic}.  \item {\bf Random Training Data.} The examples below illustrate that Theorems \ref{thrm: informal} - \ref{thrm: informal2} may produce qualitatively different phase diagrams by the division into inflationary and deflationary cases, and one might ask what structure holds for `typical' $\{s_i\}$.  A simple model for `typical' data would be to take the data $\{s_i\}$ to be independent draws from some given compactly supported measure $\nu$ under some auxiliary probability measure $P$. In this case, for a fixed $\lambda_0$, the properties of the dynamics given by Theorems \ref{thrm: informal}-\ref{thrm: informal2} are qualitatively the same, with high $P$-probability, as those when every discrete sum over $s_i$ is replaced by a $\nu$-integral. In particular, neither structure of the phases is `typical' without making assumptions on $\nu$, which do not appear to be canonical. This is also the reason why we allow potentially non-uniform sampling probabilities $p_i$, as it allows a more parsimonious approximation to non-uniform $\nu$. \end{enumerate} 
		\end{remark} In the remainder of this section, we give some preliminary calculations, following \cite{lewkowycz2020large,zhu2022quadratic}, which allow an overview of the proof of Theorem \ref{thrm: informal} and signpost the key difficulties. Taking $\varphi(w)=w$ to be the linear activation, the update rule \eqref{eq: SGD} produces, for the individual parameters $a_r(t), w_r(t)$, \begin{equation} 		\begin{cases} \label{eq: SGD 1}
			w_r(t+1)=w_r(t)-\frac{\eta}{\sqrt{n}}\mu(t)a_r(t)s_{i(t+1)}^2; \\[2ex] a_r({t+1})=a_r(t)-\frac{\eta}{\sqrt{n}}\mu(t)w_r(t)s_{i(t+1)}^2.
		\end{cases}
	\end{equation} A simple computation shows that the evolution equations for $\mu(t), \lambda(t)$ close and involve no individual weights: \begin{equation}\label{eq: SGD 2} \begin{split} &\mu(t+1)=\left(1 -\eta \lambda(t) s_{i(t+1)}^2  + \frac{\eta^2 s_{i(t+1)}^4 \mu(t)^2}{n}\right)\mu(t); \\ & \lambda(t+1) =\lambda(t) +\frac{\mu(t)^2 \eta }{n}(\eta \lambda(t)s_{i(t+1)}^4-4s_{i(t+1)}^2). \end{split}
	\end{equation} Since the total loss is, up to a fixed constant, $\ell(\Theta(t))\sim \mu(t)^2$, a spike in the loss at threshold $L$ is equivalent to a spike in the prediction $|\mu(t)|$ reaching $\sim \sqrt{L}$ and, in the sequel, we will phrase all results in terms of $|\mu(t)|$ becoming large. As a result of the standing hypothesis $\lambda_0<\lambda^{\rm MB}_{\rm max}$,  an easy inductive proof shows that $\lambda(t)$ is monotonically decreasing, and we estimate the decrease above and below by \begin{equation}  0\le a_-\frac{\eta |\mu(t)|^2}{n} \le \lambda(t)-\lambda(t+1) \le a_+\frac{\eta |\mu(t)|^2}{n};\end{equation} \begin{equation}
		 a_-:=\min_i(4s_i^2-\eta \lambda_0s_i^4);\qquad a_+:=\max_i(4s_i^2). \label{eq: decrease in curvature} 
	\end{equation} In the sequel, we reserve $a_\pm$ for these data-dependent constants. For the prediction $\mu(t)$, we neglect the higher-order term and the update to $\lambda(t)$ to find \begin{equation} \label{eq: find the RW approximately} \mu(t)\approx (1-\eta \lambda_0s^2_{i(t)})(1-\eta \lambda_0s^2_{i(t-1)})\dots (1-\eta \lambda_0s^2_{i(1)}) \mu_0\end{equation} again valid until either $\mu(t)\sim \sqrt{n/\eta}$ or when the approximation $\lambda(t)\approx \lambda_0$ breaks down. Taking the logarithm, we find the partial sums \begin{equation}
			\label{eq: partial sum} \log |\mu(t)| \approx \log|\mu_0|+\sum_{u\le t}\log|1-\eta \lambda_0 s^2_{i(u)}|
		\end{equation}  of independent and identically distributed random variables with mean $G(\lambda_0)$ given by \eqref{eq: introduce G}. As a result, the classical convergence theorems of probability produce a {\em tri}chotomy of possible behaviours: \begin{enumerate}[label=\alph*).]
		\item If $G(\lambda_0)>0$, then the law of large numbers applied to \eqref{eq: partial sum} forces $|\mu(t)|$ to eventually become large, so a spike is guaranteed; \item If $G(\lambda)<0$ and $\eta \lambda s_\star^2>2$, then large values of $|\mu(t)|$ correspond to {\em large deviations} of the sum \eqref{eq: partial sum}. As a result, spikes of size $M\lesssim \sqrt{n/\eta}$ occur with probabilities $\sim M^{-\vartheta}$, where the Cram\'er exponent $\vartheta>0$ is given by \eqref{eq: def vartheta informal}; \item If $\eta \lambda s_\star^2<2$, then every term in the sum \eqref{eq: partial sum} is non-positive and $|\mu(t)|$ is monotonically decreasing. No spike is possible. 
	\end{enumerate}  We refer to the cases a) and b), distinguished by the sign of the log-drift $G(\lambda_0)$ as the {\em inflationary} and {\em deflationary} regimes of the catapult phase, and c) as the monotone phase. In case b), the exponent $\vartheta$ has the property that, treating the approximation \eqref{eq: find the RW approximately} as exact, $|\mu(t)|^\vartheta$ is a martingale, which produces the asymptotic hitting probability claimed by closely following the upper and lower bounds of Cram\'er's theorem, see, for example, \cite[Proof of Theorem 2.2.3, pp. 30-33]{dembo2009large} or \cite[Theorem I.3]{hollander2000large}. When the prediction $\mu(t)$ ever reaches the level $\mu(t)\sim \sqrt{n/\eta}$, then the updates \eqref{eq: decrease in curvature} become significant, and the spike allows a reduction in the curvature. \\\\ {\bf Key Challenge of a Rigorous Proof.} The key difficulty extending the argument sketched above to the full update rule \eqref{eq: SGD 2} is to account for the fact that $\lambda(t)$ will vary in time according to \eqref{eq: decrease in curvature}. While the updates to the kernel {\em at each time step} are negligible while $|\mu(t)|\ll \sqrt{n/\eta}$, it is {\em a priori} possible that it could fall significantly without $|\mu(t)|$ ever becoming large, either by a larger number of smaller spikes at height $M\ll \sqrt{n/\eta}$, or over very large timescales $T\gg \log({n/\eta})$. We highlight two further technical ideas which are needed to address this challenge. Firstly, Proposition \ref{prop: slow escape} shows that the probability of the kernel changing by some threshold $\epsilon$ while the prediction remains below $M$ decays like \begin{equation}\label{eq: preview slow escape} \lesssim \exp\left(-\epsilon \gamma \left(\frac{n}{\eta M^2}\right)^\beta\right); \qquad \beta, \gamma>0. \end{equation} This makes the claim 3) in the summary rigorous, and motivates us to restrict to scales \begin{equation} \label{eq: restrict to scales} M\lesssim \frac{1}{\log^{1/\beta}(n/\eta)}\sqrt{\frac{n}{\eta}} \end{equation} in order to ensure that the error probability decays like $(n/\eta)^{\vartheta/2}$. In this way, neglecting the updates to the kernel in \eqref{eq: find the RW approximately} requires us to restrict to scales which are smaller (by a polylogarithmic correction) than the scale $|\mu(t)|\sim \sqrt{n/\eta}$ required to neglect the higher-order term in \eqref{eq: SGD 2}. We therefore distinguish between `moderate spikes', up to such a scale, and `large spikes' on scales which rapidly decrease the curvature.  \\\\ The second idea which is necessary to overcome the challenge is that, in order to maintain the asymptotic $(|\mu_0|/M)^\vartheta$, it is necessary to apply the previous estimate \eqref{eq: preview slow escape} repeatedly: If $|\mu_0|$ is very small, it may happen that $(n/\eta)^{\vartheta/2}\gg (|\mu_0|/M)^\vartheta$. In order to avoid this, we show how the range $[|\mu_0|,M]$ may be divided into a (potentially large) number of scales $$L_0=M\gg L_1\gg \dots\gg L_k=|\mu_0|$$ in such a way that the probabilities decreasing the kernel by $\frac{\epsilon}{2^j}$ between scales $L_{j+1}, L_j$ has the correct asymptotic $\le c_j(|\mu_0|/M)^\vartheta$, with a summable prefactor $\sum_j c_j<\infty$. The error probabilities may therefore be recombined without changing the overall behaviour to produce the asymptotic probability claimed.  \\\\ As a result of the first point, we divide the analysis between `moderate' and `large' spikes. The strategy for making statements about large spikes will be to consider first the probability of reaching the largest scale $M$ \eqref{eq: restrict to scales} allowed, and then considering the probability of a kernel reduction conditional on having reached a moderate spike. For this reason, the theorem statements in Theorems \ref{thrm: LLN large} - \ref{thrm: LDP large spikes} express only the conditional probabilities, and the reassembly to the claims of Theorem \ref{thrm: informal} follows using the strong Markov property and the corresponding results for moderate spikes. \edit{\subsection{Implications for Curriculum Learning} \label{sec: curriculum learning} The above discussion, and its rigorisation in Theorems \ref{thrm: informal}-\ref{thrm: informal2}, produce two views on {\em curriculum learning}, where the way in which different datapoints are sampled is updated in an on-line way over the course of the training. Since the presence of spikes ultimately leads to convergence to minima with lower curvature, it is advantageous to construct a curriculum in such a way as to increase their probability.\\ Firstly, the probability of spikes is increased by using the same draw $i(t)$ repeatedly, rather than resampling. This motivates,  for some $k\in \NN$, using the same data point $k$ times in a row, so that $i(kt+1)=i(kt+2)=\dots=i(kt+k)$ for each $t\ge 0$, while $(i(tk+1): t\ge 0)$ are fresh, independent samples. In this case, the same calculations leading to \eqref{eq: find the RW approximately} now produce $$ \mu(kt)=(1-\eta \lambda_0 s^2_{i(k(t-1)+1)})^k\dots (1-\eta \lambda_0 s^2_{i(1)})^k \mu_0$$ thus replacing each multiplicative factor with its $k^{\rm th}$ power. Viewing all $k$ updates as the single step of a Markov chain, the effective log-drift $G^{\rm eff}_k(\lambda)=kG(\lambda)$ has the same sign as the original SGD, but the exponent in the deflationary case is changed to $$\vartheta^{\rm eff}_k(\lambda)=\sup\left\{\theta\ge 0: \sum_{i=1}^m p_i (|1-\eta \lambda s_i^2|^k)^\theta\le 1\right\} =\frac{\vartheta(\lambda)}{k}. $$ Correspondingly, choosing $k$ of the scale $\sim \log(n/\eta)$ - which in practical applications is not excessively large - produces spike probabilities which do not decay with $n, \eta$, thus improving the probability of converging to a minimum with lower curvature. \\\\ Secondly, the machinery of large deviations also allows us to identify an {\em entropically minimal curriculum with guaranteed spikes}. It may also be advantageous to determine a curriculum which is, in some sense, a minimal alteration of the original sampling law $\{p_i\}$. Measuring the `size' of an alteration through the {\em entropic cost} of a (random, time-dependent) $q=(q_i(t): i\in [m], t\ge 0)$, defined by the expected, time-summed Kullback-Leibler divergence $${\rm Ent}(q|p):=\EE\left[\sum_{t\ge 0} \sum_i q_i(t)\log \frac{q_i(t)}{p_i}\right] $$ it may be shown that the probabilities of spikes in Theorems \ref{thrm: informal}-\ref{thrm: informal2} decay like \begin{equation*}\begin{split} &- \log \PP(\text{Loss reaches } L) \\ & \hspace{2cm}\sim \inf\left\{{\rm Ent}(q|p): \text{ Loss reaches $L$ with high probability under curriculum }q\right\} \end{split}\end{equation*} and similarly for other types of events. Further, the proofs of the large deviations estimates in Section \ref{sec: LDP} identify a (nearly) optimal $q$ in epochs between spikes: Given $\lambda_0 \in (\lambda^{\rm MB}_{\rm crit}, \lambda^{\rm MB}_{\rm max})$, the equation $$ \sum_{i=1}^m p_i |1-\eta \lambda_0 s_i^2|^\theta=1$$ has a unique positive root $\vartheta$, which may be computationally found by interval bisection, and the entropically optimal $q$ is given by $q_i=|1-\eta \lambda_0 s_i^2|^\vartheta p_i$. In practice, if $\lambda_0$ is not known or difficult to calculate, the multiplicative factors may be approximated by proxies $$  |1-\eta \lambda_0 s_i^2|\approx \left(\frac{\ell(\Theta(1): i(1)=i)}{\ell(\Theta(0))}\right)^{1/2}.$$ This therefore produces a computationally tractable way to find a curriculum which ensures convergence to a lower-curvature minimum with close to minimal entropic cost.} \paragraph{Organisation of the Paper} In the remainder of this section, we provide simple examples which illustrate the rich variety of structures derived in Theorem \ref{thrm: informal}. Section \ref{sec:litreview} is a literature review and discussion of connections of the present work to the literature. The remainder of the paper is devoted to the rigorisation of Theorem \ref{thrm: informal}, starting with preliminaries and definitions in constant use in Section \ref{sec: preliminaries}. The precise statements and proofs are divided by the four combinations of moderate or large spikes, and inflationary or deflationary drift: Section \ref{sec: LLN} deals with moderate spikes in the inflationary case, Section \ref{sec: LDP} with moderate spikes in the deflationary case, Section \ref{sec: large LLN} with both cases for large spikes. Finally, Section \ref{sec: RELU} establishes a precise connection between the case of linear activation sketched above, and the case of ReLU activation, which shows how the conclusions may be extended to that case.
	 
	  \subsection{Examples of the Theorem} \label{sec: examples}

In this section, we will present some examples to illustrate the theory derived in the previous section. Even in very simple settings, where the dataset consists of only two or three datapoints chosen with (potentially) unequal probabilities, the structure of $G, \alpha$ shows non-trivial structure. \begin{wrapfigure}[20]{l}{0.4\textwidth}
 \centering
  \begin{tikzpicture}
\pgfmathsetmacro{\eta}{1}
\pgfmathsetmacro{\p}{0.5}
\pgfmathsetmacro{\x}{1.3}

\definecolor{convblue}{RGB}{144,198,233}   
\definecolor{convblue2}{RGB}{173,222,245}  
\definecolor{divred}{RGB}{231,150,150}     
\definecolor{stripegreen}{RGB}{33, 186, 34}

\pgfmathsetmacro{\lamThree}{2/(\eta*1.3*1.3)}        
\pgfmathsetmacro{\lamTwo}{1/\eta}                  
\pgfmathsetmacro{\lamOne}{2/\eta}                  
\pgfmathsetmacro{\lamG}{1.591}                      
\pgfmathsetmacro{\lamF}{2.2}                    

\begin{axis}[
  width=6cm,height=4cm,
  axis lines=middle,
  xmin=0,xmax=2.366,
  ymin=-3,ymax=1,
  xlabel={$ \lambda $},
  ylabel={$(i)$},
  domain=0:2.366,
  samples=400,
  thick,
  xtick={0,\lamThree,\lamTwo,\lamOne},
  xticklabels={},
  ytick={-4,-3,-2,-1,0,1,2,3},
  every tick label/.append style={font=\small},
]

\

\path[fill=convblue, fill opacity=0.35, draw=none]
  (axis cs:0,-4) rectangle (axis cs:\lamThree,3);

\path[fill=convblue2, fill opacity=0.35, draw=none]
  (axis cs:\lamThree,-4) rectangle (axis cs:\lamG,3);

\path[fill=stripegreen, fill opacity=0.25, draw=none]
  (axis cs:\lamG,-4) rectangle (axis cs:2.366,3);

\addplot[very thick,blue]
  { 0.5*ln(abs(1-\eta*1.3*1.3*x))
    + (1-0.5)*ln(abs(1-\eta*x)) };

\addplot[black,thick,dashed]   coordinates {(\lamThree,-4) (\lamThree,1)};
\addplot[black,thick,dashed]   coordinates {(\lamG,-4)  (\lamG,1)};




\pgfmathsetmacro{\yTopOne}{-1.5}  
\pgfmathsetmacro{\yTopTwo}{2.} 

\draw[<->] (axis cs:0,\yTopOne) --
           node[above]{c}
           (axis cs:\lamThree,\yTopOne);

\draw[<->] (axis cs:\lamThree,\yTopOne) --
           node[above]{b}
           (axis cs:\lamG,\yTopOne);

\draw[<->] (axis cs:\lamG,\yTopOne) --
           node[above]{a}
           (axis cs:2.366,\yTopOne);




\end{axis}
\end{tikzpicture}


\begin{tikzpicture}
\pgfmathsetmacro{\eta}{1}
\pgfmathsetmacro{\p}{0.83}
\pgfmathsetmacro{\x}{sqrt(3)}

\pgfmathsetmacro{\lamThree}{2/(\eta*\x*\x)}        
\pgfmathsetmacro{\lamTwo}{1/\eta}                  
\pgfmathsetmacro{\lamOne}{2/\eta}                  
\pgfmathsetmacro{\lamrootone}{0.794}				   
\pgfmathsetmacro{\lamroottwo}{0.950}	               
\pgfmathsetmacro{\lamrootthree}{1.028}              
\pgfmathsetmacro{\lamG}{1.12}                      
\pgfmathsetmacro{\lamC}{2/((\p*\x*\x+1-\p)*\eta)}  
\pgfmathsetmacro{\lamF}{2.2}                    


\definecolor{convblue}{RGB}{144,198,233}   
\definecolor{convblue2}{RGB}{173,222,245}  
\definecolor{divred}{RGB}{231,150,150}     
\definecolor{stripegreen}{RGB}{33, 186, 34}

\begin{axis}[
  width=6cm,height=4cm,
  axis lines=middle,
  xmin=0,xmax=4/(\eta*\x*\x),
  ymin=-2,ymax=1,
  xlabel={$ \lambda $},
  ylabel={${ (ii)}$},
  domain=0:4/3,
  samples=400,
  thick,
  xtick={0,\lamThree,\lamTwo,\lamOne},
  xticklabels={},
  ytick={-4,-3,-2,-1,0,1,2,3},
  every tick label/.append style={font=\small},
]

\addplot[very thick,blue]
  { \p*ln(abs(1-\eta*3*x))+(1-\p)*ln(abs(1-\eta*x)) };

\addplot[black,thick,dashed]   coordinates {(\lamThree,-4) (\lamThree,2.5)};

\addplot[black,thick,dashed]           coordinates {(\lamF,-4) (\lamF,2.5)};
\addplot[black,thick,dashed]           coordinates {(\lamrootone,-4) (\lamrootone,2.5)};
\addplot[black,thick,dashed]           coordinates {(\lamroottwo,-4) (\lamroottwo,2.5)};
\addplot[black,thick,dashed]           coordinates {(\lamrootthree,-4) (\lamrootthree,2.5)};




\path[fill=convblue, fill opacity=0.35, draw=none]
  (axis cs:0,-4) rectangle (axis cs:\lamThree,3);

\path[fill=convblue2, fill opacity=0.35, draw=none]
  (axis cs:\lamThree,-4) rectangle (axis cs:\lamrootone,3);

\path[fill=stripegreen, fill opacity=0.25, draw=none]
  (axis cs:\lamrootone,-4) rectangle (axis cs:\lamroottwo,3);
  
 \path[fill=convblue2, fill opacity=0.35, draw=none]
  (axis cs:\lamroottwo,-4) rectangle (axis cs:\lamrootthree,3);

\path[fill=stripegreen, fill opacity=0.25, draw=none]
  (axis cs:\lamrootthree,-4) rectangle (axis cs:4/3,3); 

\pgfmathsetmacro{\yTopOne}{-1.5}  
\pgfmathsetmacro{\yTopTwo}{-2.} 
\pgfmathsetmacro{\ybotone}{-3}

\draw[<->] (axis cs:0,\yTopOne) --
           node[above]{c}
           (axis cs:\lamThree,\yTopOne);

\draw[<->] (axis cs:\lamThree,\yTopOne) --
           node[above]{b}
           (axis cs:\lamrootone,\yTopOne);  

\draw[<->] (axis cs:\lamrootone,\yTopOne) --
           node[above]{a}
           (axis cs:\lamroottwo,\yTopOne);

\draw[<->] (axis cs:\lamroottwo,\yTopOne) --
           node[above]{b}
           (axis cs:\lamrootthree,\yTopOne); 
           
\draw[<->] (axis cs:\lamrootthree,\yTopOne) --
           node[above]{a}
           (axis cs:4/3,\yTopOne);




\end{axis}
\end{tikzpicture}

\caption{Plots of $G(\lambda)$ for the examples (\ref{eq: example 1} - \ref{eq: example 2})}
  \label{fig:GPlot}
\end{wrapfigure}
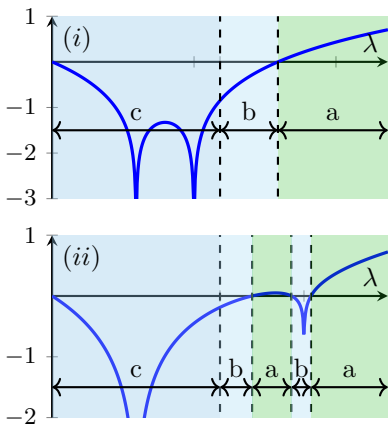  \subsubsection{The Phases with Two Datapoints} \label{ex: two data points}  We begin by discussing the structure of $G$ and the different parameter regimes where the dataset consists only of two points. In Figure \ref{fig:GPlot}, the log-drift function $G(\lambda)$ is plotted for the full range $0\le \lambda \le 4/{\eta s_\star^2}$ with $\eta=1$ for the two illustrative cases \begin{equation} \label{eq: example 1}
	 {\rm i)} \quad \{(s_i, p_i)\}=\{(1, 0.5), (1.3, 0.5)\}
\end{equation}\begin{equation} \label{eq: example 2}
 {\rm ii)} \quad \{(s_i, p_i)\}=\{(1, 0.83), (\sqrt{2}, 0.17)\}.
\end{equation} In each case, the different ranges of $\lambda$ produced by the discussion under \eqref{eq: partial sum} are labelled, and colour-coded consistently with Figure \ref{fig: phases intro}. We see that the function $G$ given by \eqref{eq: introduce G} is never monotone and neither convex nor concave. Moreover, the non-monotonicity can result in multiple sign changes of $G$, illustrated in case ii), meaning that the distinction between the phases a and b is not necessarily monotone as a function of $\lambda$.
		
		
\subsubsection{Comparison of Critical Values to Full-Batch Critical Values} We next compare the parameter ranges in Theorem \ref{thrm: informal} to the critical values $\lambda_{\rm crit}, \lambda_{\rm max}$ for the full-batch dynamics, applied to the loss $$ \ell(\Theta)=\sum_{i=1}^m L(F(\Theta, s_i), y_i)=|\mu(t)|^2\sum_{i=1}^m p_i s_i^2 =:|\mu(t)|^2 \bar{s}^2 $$ where we write $\bar{s}^2:=\EE s_{i(1)}^2$. Following the analysis of \cite{lewkowycz2020large,zhu2022quadratic}, the critical value of $\lambda$ separating the convergent from catapult phases, is \begin{equation}
	\lambda^{\rm FB}_{\rm crit}=\frac{2}{\eta \bar{s}^2}.
\end{equation} By definition, $\bar{s}\le s_\star$, with an equality only in the degenerate case where $|s_i|=s_\star$ for each $i$. As a result, it follows that the lower bound of the SGD catapult phase is \begin{equation}
	\lambda^{\rm MB}_{\rm crit}:=\frac{2}{\eta s_\star^2}\le \lambda^{\rm FB}_{\rm crit}.
\end{equation} From Theorem \ref{thrm: informal}, we conclude that, aside from the degenerate case, large spikes remain possible and are, in the worst case, polynomially (un)likely, in a nonempty region $\lambda\in (\lambda^{\rm MB}_{\rm crit}, \lambda^{\rm FB}_{\rm crit})$ where they are not possible for full-batch gradient descent. \\ \\ It is not generally possible to make a link between the catapult and monotone phases of full-batch gradient descent, and the inflationary and deflationary phases. To see this, we note that the condition for catapults in the full-batch case may be written more suggestively as \begin{equation}
	\log|1-\eta \lambda \bar{s}^2|=\log\left|1-\eta \lambda \mathbb{E}s_{i(1)}^2\right|>0
\end{equation} which closely resembles the condition for the inflationary regime \begin{equation}
	\EE\left[\log|1-\eta\lambda s_{i(1)}^2|\right]>0.
\end{equation} However, neither of the two displays above implies the other. While the function $u\mapsto \log u$ is concave on its domain, the function $\log |u|$ is neither globally convex nor globally concave, and no application of Jensen's inequality is possible. In the next two examples we give explicit examples to show that both possible implications between the above displays are false. Correspondingly, it is possible for the same parameters to be in the inflationary phase for SGD but the monotone phase for GD, or in the catapult phase for GD and the deflationary phase for SGD. Hence, for a given dataset and learning rate $\eta$, the presence of a spike for GD neither implies, nor is implied by, a high-probability spike for SGD.
\subsubsection{The inflationary regime without full-batch catapults} \label{ex: SGD catapults}  Consider the two-point example with $\eta=0.3$, $s_1=1, p_1=0.6, s_2=3, p_2=0.4$, started from initial curvature $\lambda_0=\frac43$. The critical curvature for the full-batch loss is \begin{equation}
	\lambda^{\rm FB}_{\rm crit}=\frac{2}{\eta \sum_{i=1}^2 p_i s_i^2}=\frac{100}{63} > \lambda_0
\end{equation} so the full-batch gradient descent (GD) is in the convergent phase. Meanwhile, the log-drift for SGD is \begin{equation}
	G(\lambda_0)=0.6 \times \log 0.6 +0.4\times \log 2.6 \approx 7.6\times 10^{-2} \end{equation} leading to an eventual large spike of the loss. 
	\subsubsection{The deflationary regime with full-batch catapults} \label{ex: full batch catapults}
	
Consider the dataset \begin{equation}{\rm iii)} \qquad \{(s_i, p_i)\}= \left\{(5.92, 0.4),(3.74, 0.6)\right\} \end{equation} with $\eta=0.1, \lambda_0=1$. The critical curvature for the full-batch loss is \begin{equation}
	\lambda_{\rm crit}=\frac{2}{\eta \sum_{i=1}^2 p_i s_i^2} \approx 0.89 \in \left( \frac{\lambda_0}2, \lambda_0\right)
\end{equation} which implies that the full-batch dynamics is in the catapult phase. On the other hand, the log-drift for SGD is \begin{equation}
	G(\lambda_0)\approx -0.18
\end{equation} and the SGD is in the deflationary regime.

\subsubsection{Non-monotonicity of Spikes as a function of Curvature} \label{ex: non monotone behaviour} The structure of the catapult phase shown in Theorem \ref{thrm: informal} shows that increasing curvature may not lead to an increased propensity of the system to spike. An example where $G$ is not monotone over $(\lambda^{\rm MB}_{\rm crit}, \lambda^{\rm MB}_{\rm max})$ was shown in Section \ref{ex: two data points} above, and which means that increasing $\lambda_0$ may \begin{enumerate}
	\item Decrease $G(\lambda_0)$, remaining in the inflationary regime, so that the large spike occurs later; or \item Move from the inflationary regime to the deflationary regime, making a large spike less likely.
\end{enumerate} Within the deflationary regime, the function $\vartheta(\lambda)$ given by \eqref{eq: def vartheta informal} is the unique non-negative root to the given equation. The implicit function theorem thus shows that it is differentiable away from singularities $\lambda=(\eta s_i^2)^{-1}$, but the resulting closed form need not have a sign. Correspondingly, a third result of increasing $\lambda$ may be \begin{enumerate}\setcounter{enumi}{2}
	\item Increase $\vartheta(\lambda)$, remaining in the deflationary regime, and so make large spikes much less likely.
\end{enumerate} Taking the dataset \begin{equation} {\rm iv)}\qquad 
	\label{eq: dataset nonmonotone alpha} \{(s_i, p_i)\}=\{(1, 0.55), (0.5, 0.34), (0.9, 0.11)\}
\end{equation} produces \begin{equation}
	\label{eq: conclusion nonmonotone alpha}\vartheta(2.60/\eta)\approx 0.380;\qquad  \vartheta(3.14/\eta) \approx 0.102; \qquad \vartheta(3.95/\eta)\approx 0.315.
\end{equation} Consequently, spikes are more likely when $\lambda\approx 3.14/\eta$ than either of the other two values. A plot of $G$ and $\vartheta$ for this dataset is displayed in Figure \ref{fig: NMA}.

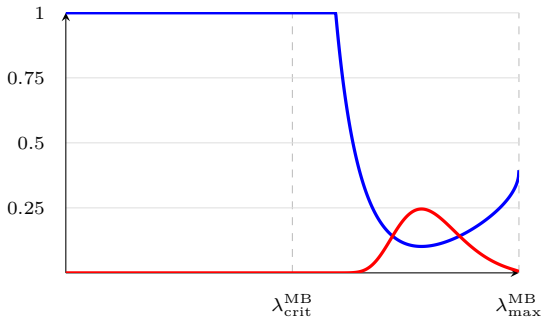
\begin{wrapfigure}{l}{0.54\textwidth}
\centering
\begin{tikzpicture}

\pgfplotstableread[col sep=comma]{ex2.csv}\datatable

\def\xmin{0}
\def\xmax{4}

\pgfplotsset{
  myclean/.style={
    axis lines=middle,
    xmin=\xmin, xmax=\xmax,
    tick align=outside,
    tick style={draw=none},
    xlabel near ticks,
    ylabel near ticks,
    label style={font=\small},
    clip=true,
  }
}

\begin{axis}[
  myclean,
  at={(0,-0.28\textheight)}, anchor=north west,
  width=0.54\textwidth,
  height=0.22\textheight,
  xlabel={},
  xmajorgrids=true,
  ymajorgrids=true,
  xtick={2,4},
  xticklabels={$ \lambda^{\mathrm{MB}}_{\mathrm{crit}}$, $ \lambda^{\mathrm{MB}}_{\mathrm{max}}$},
  x grid style={black!25, dashed},
  ytick={0.25,0.5,0.75,1},
  yticklabels={0.25,0.5,0.75,1},
  y grid style={black!12},
  every tick label/.append style={font=\scriptsize},
]

\addplot[very thick, blue] table[x=lambda, y=vartheta_cap] {\datatable};

\addplot[very thick, red]
  table[
    x=lambda,
    y expr=exp(-0.5*ln(1e12)*\thisrow{vartheta_cap})
  ] {\datatable};

\end{axis}

\end{tikzpicture}

\caption{$\max(1,\vartheta(\lambda))$ (blue) and $n^{-\vartheta(\lambda)/2}$ with $n=10^{12}$ (red) for the dataset \eqref{eq: dataset nonmonotone alpha}.}
\label{fig: NMA}
\end{wrapfigure}

\subsection{Literature Review and Discussion}\label{sec:litreview}

\paragraph{\em Motivation of the Model: High Dimensionality and the Ubiquity of Spikes} We first give a motivation of the models described by Theorem \ref{thrm: informal}, as well as an informal argument that it partially explains the ubiquity of spikes in practice. In addition to the large numbers of parameters, the dimension $d$ of the domain of a learned function may also be in the millions or billions. In various models studied in the literature, the dynamics in the different directions completely or partially decouple and evolve independently of one another over the training process, for example in SGD-related methods in matrix factorisation \cite{gemulla2011large}, see also \cite[Proposition 5]{liang2025gradient} for a similar observation. The same also applies to Mixture-of-Expert type architectures \cite{dai2022stablemoe,roller2021hash,qiu2024layerwise,seo2025mofe,shazeer2017outrageously,fedus2022switch} if the routing coefficients are kept constant after initialisation, as is done in practice in \cite{fei2024scaling,team2025every}. \\ \\ The models in Theorem \ref{thrm: informal} are chosen as a caricature of the per-channel dynamics of such a model. Globally, we imagine that choosing any datapoint ${\bf x}_i$ from the whole training dataset only affects the parameters of one channel $1\le \alpha(i)\le d$, producing a loss through a subnetwork of one of the forms in Theorem \ref{thrm: informal} through a scalar $s_i=\langle e_\alpha, {\bf x}_i\rangle$. Even if the input data ${\bf x}_i$ are normalised, $s_i$ will (in general) not be, and it is possible for multiple datapoints to contribute to each channel. Theorem \ref{thrm: informal} now governs the probability of a spike in each channel $\alpha$, with a decay rate no worse than $(n/\eta)^{-\vartheta_\alpha/2}$ as soon as \eqref{eq: minibatch critical values} is satisfied.  \\\\ We now consider the evolution of the loss of the whole network. If, for some $\theta\ge 0$, a significant fraction of the channels have $\vartheta_\alpha \le \theta$, then the total number of spikes is approximately a Poisson random variable, with mean of the order $\mathcal{O}(d(n/\eta)^{-\theta/2})$. In particular, since $\theta>0$ can be small, a large number of spikes can be caused without imposing any polynomial relation on $d, n$.
\paragraph{\em Convergence of SGD, Catapult Mechanism}
{
The basic analysis of SGD, see for example \cite{ljung2003analysis,lan2020first}, shows that, for sufficiently small $\eta$, all but a vanishing fraction of the time is spent near critical points. While the landscape of $\ell$ may contain many saddle points \cite{dauphin2014identifying}, it is known that SGD avoids, with high probability, spurious critical regions, see for example \cite{mertikopoulos2020almost} and references therein. The appearance of large spikes in the loss of SGD has been attributed empirically to the catapult mechanism \cite{zhu2023catapults}, and as a jumping-off point, we take the models discussed in \cite{lewkowycz2020large,zhu2022quadratic} to understand how the catapult mechanism interacts with SGD noise. In this setting, the results presented in Theorem \ref{thrm: informal} above allow us to rigorously validate many of the results of \cite{zhu2023catapults}, while allowing refinements of others, see Example \ref{ex: non monotone behaviour}. Both the informal analysis given above and the rigorous proofs can also be extended to cover the quadratic networks studied in \cite{zhu2022quadratic} without significantly changing the analysis, replacing \eqref{eq: SGD 2} above by \cite[Equations (11-12)]{zhu2022quadratic}.\\ \paragraph{\em Implicit bias, edge of stability, and sharpness.}
Beyond convergence, recent literature has sought to understand which {\em implicit bias} mechanisms of gradient methods help methods find minima with additional structural properties, see \cite{pesme2021implicit, vasudeva2025rich} and references therein. Blanc et al. \cite{blanc2020implicit}, Li et al. \cite{li2021happens}, Shalova et al. \cite{shalova2024singular} and Marion-Chizat \cite{marion2024deep} showed that, on a sub-manifold of minima, SGD will further select points for smaller curvature, typically on a longer timescale than the original convergence to the minima. Recent works have also identified the {\em edge of stability} effect \cite{andreyev2024edge,arora2022understanding,damian2022self,wu2022alignment}, where  gradient descent algorithms tend to sharpen curvatures below the critical value $2/\eta$, and reduce curvatures above this value. \\ \paragraph{\em Lazy training and the NTK limit.} One of the insights driving recent breakthroughs on proving the convergence of gradient descent algorithms \cite{du2018gradient,li2018learning} is the observation that, when $F$ is a wide neural network, each component $\theta_i, 1\le i\le n$ will hardly vary over the evolution in the limit $n\to \infty$, a phenomenon now called {\em lazy training} \cite{chizat2019lazy}. As a result, the whole network behaves like a purely linear model around its linearisation \cite{jacot2018neural}. Consistently with this, the models analysed here behave, in epochs between spikes, close to purely linear models. The (im)probability of the spikes is determined by the random dynamics of the linear part, while the spikes themselves allow the system to (briefly) leave the lazy regime and show nonlinear effects.\\ \paragraph{\em Diffusion and large-deviation perspectives.}
{
An important step in understanding the typical behaviour of SGD is to view it as the discretisation of a stochastic differential
equation, going back to Li \cite{li2017stochastic}. This insight led to a heuristic explanation for SGD finding flatter minima via large deviation theory by \cite{ibayashi2021quasi}, made precise by Azizian et al. \cite{azizian2024long,azizian2025global}. The viewpoint of the current work is complementary to the works \cite{azizian2024long,azizian2025global}, which show a preference for minimisers of a certain quasi-potential in the limit $\eta\to 0$ with dimension parameter $n$ held fixed, while we consider the opposite sort of limit $n\to \infty$ with $\eta$ held fixed. In both cases, the transitions between the basins of attraction of different minima are driven by large deviations and produces short time intervals where the loss is large, although driven in the two cases by different mechanisms. The models of networks we use for tractable, exact analysis are similar to, and slightly more complicated than, diagonal networks, where the benefit of stochasticity was identified by Pesme \cite{pesme2021implicit,even2023s}.\\ \paragraph{\em The Criterion for Instability \& The Edge of Stochastic Stability} {The criterion derived heuristically above for distinguishing the inflationary and deflationary regimes in terms of $G(\lambda)$ is closely related to the {\em Lyapunov exponents} of the linearised dynamics, which were investigated by Chemnitz-Engel \cite{chemnitz2025characterizing}. We distinguish the positivity of $G$, which drives a large spike of the nonlinear dynamics,} and {\em blowup in expectation}, where the mean of the process diverges. An analysis of the latter condition \cite{agarwala2024high} leads to a condition called the {\em stochastic edge of stability},\edit{ see also \cite[Appendix A]{chemnitz2025characterizing}}. This distinction is not, of course, specific to SGD, and may be illustrated by whether one considers the process $$ Z_t:=\exp\left(B_t-\frac{t}4\right),$$ for $B$ a standard Brownian motion, to be large or small at large times $t$. For the purposes of understanding the ubiquity of spikes, only the almost sure condition is relevant, since measuring largeness in expectation counts potentially large contributions from events only observed on a very small proportion of realisations. Indeed, this naturally poses the question of how unlikely such large values are, in the cases where they are not guaranteed, which is answered by the LDP analysis underlying the deflationary cases of Theorem \ref{thrm: informal}. In fact, the edge of stochastic stability criterion reappears elsewhere in the analysis\footnote{Precisely: The criterion for the edge of stochastic stability is equivalent to $\vartheta(\lambda)\le 2$, which determines whether we may take $\beta=2$ or $\beta<2$ in Proposition \ref{prop: slow escape}.}, in determining the scale of poly-logarithmic adjustments to the scale $\sqrt{n/\eta}$ below which `slow kernel decrease' contributes negligible probabilities, see Proposition \ref{prop: slow escape} and the proof of Theorem \ref{thm: LDP spikes}. 

 \paragraph{\em How Spikes End} A byproduct of the analysis, which is relevant in quantifying the reduction in the curvature before and after a spike, is to identify the ways in which spikes {\em end} after a large threshold has been reached. We show that three distinct effects are possible: gradual curvature reduction beyond a critical threshold; neuron deactivation in models with ReLU activation; or spike collapse, where nonlinear effects end the spike in a single step. 
  \paragraph{\em Phases of Curvature vs Phases of Learning Rate} The distinct phases and their phenomena can be interpreted in two different ways. In \cite{lewkowycz2020large,zhu2022quadratic}, the different phases are described by the learning rate $\eta>0$, whereas we refer to them as depending on curvature $\lambda$. In fact, all of the quantities $G(\lambda)=G(\lambda; \eta, \{\text{data}\})$, $\vartheta(\lambda)=\vartheta(\lambda;\eta,\{\text{data}\})$ depend only on $\lambda, \eta$ through the product $\lambda\eta$, which allows an exact translation of the two viewpoints. We take the viewpoint that $G, \vartheta$ are functions of $\lambda$, which determine how the curvature $\lambda(t)$ will vary over the course of the training. The alternative viewpoint would be to describe Theorems \ref{thm:lln moderate} - \ref{thm: LDP spikes} as describing, for a given initial curvature $\lambda_0$, which learning rates $\eta$ make curvature reductions via a spike likely, or at least not too unlikely.

\section{Preliminaries} \label{sec: preliminaries} For the remainder of the paper, we consider a dataset $\{(s_i, p_i)\}$ fixed forever. We define, for $1\le i\le m$, \begin{equation} \label{eq: def bi}
		b_i(\lambda):=|1-\lambda \eta s_i^2|
	\end{equation} which are the multiplicative factors appearing in the approximation \eqref{eq: find the RW approximately}. For upper and lower bounds we define, for thresholds $\epsilon>0, \kappa>0$, and every $n$,\begin{equation}
		\label{eq: def bi+} b_i^+:=b_i^+(\epsilon,\kappa,\lambda)=b_i+\eta s_\star^2 \epsilon + \eta s_\star^4 \kappa^2; 
	\end{equation} \begin{equation}
		\label{eq: def bi-} b_i^-:=b_i^-(\epsilon,\kappa,\lambda)=\max\left\{b_i-\eta s_\star^2 \epsilon - \eta s_\star^4 \kappa^2, 0\right\}. 
	\end{equation} These are constructed so that, whenever $0\le |\mu|\le \kappa\sqrt{n/\eta}$ and $\lambda_0-\epsilon \le \lambda \le \lambda_0$, we have the two-sided bound on the multiplicative factor appearing in \eqref{eq: SGD 2}: \begin{equation} \label{eq: multiplicative ublb} b_i^-(\epsilon,\kappa,\lambda_0)\le \left|1-\eta \lambda s_i^2+\frac{\eta s_i^4 \mu^2}{n}\right|\le b^+_i(\epsilon,\kappa,\lambda_0).\end{equation}	
The probabilistic analysis rests on the use of tools from martingale theory and the construction of (sub-, super-) martingales. We refer to \cite{williams1991probability} for an exposition. We equip the probability space $(\Omega, \cF, \PP)$ on which the SGD is defined with the filtration \begin{equation}
 	\cF(t):=\sigma\left(a_r(0), w_r(0), i(s): 1\le r\le n, 1\le s\le t\right).
 \end{equation}  
In the sequel, all filtration-dependent properties, such as adaptedness or the martingale property, will implicitly be taken with respect to this filtration, and we will not reiterate the filtration. For thresholds $M$, we define the stopping time $\TPM$ where the prediction first exceeds $M$ in absolute value: \begin{equation}
		\label{eq: define taupred linear model} \TPM:=\inf\{t\ge 0: |\mu(t)|\ge M\}
	\end{equation} and, for $\lambda<\lambda_0$, we define the stopping time $\TK{\lambda}$ when the kernel falls below $\lambda$:\begin{equation}
		\label{eq: define tauK linear model} \TK{\lambda}:=\inf\{t\ge 0: \lambda(t)<\lambda\}.
	\end{equation} Whenever $Z=(Z(t), t\ge 0)$ is an adapted process and $T$ is a stopping time, we write $Z^T$ for the stopped process $Z^T(t)=Z(t\land T), t\ge 0$.

\section{Moderate Spikes in the Log-Inflationary Regime} \label{sec: LLN} We now give the precise result on the appearance of moderate spikes in the inflationary regime $G(\lambda_0)>0$. 
\begin{theorem}\label{thm:lln moderate}
In the setting of Assumption \ref{hyp: assumption linear model}, assume further that $G(\lambda_0)>0$. Then, for every $0<\delta<1$, there exist constants $\vartheta, \kappa, \gamma>0$, depending on $\delta, \lambda_0$ and $\{(s_i, p_i)\}$  such that, for all $n, \eta$ with $n/\eta\ge e$ and for every threshold \begin{equation} \label{eq: threshold choice LLN} M\in\left[\h |\mu_0|,\kappa\sqrt{\frac{n}{\eta}}\h \right]\end{equation} it holds that\[
\mathbb{P}_{\mu_0, \lambda_0}\!\left(\frac{\TPM}{\log(M/|\mu_0|)} \notin \left[\frac{1-\delta}{G(\lambda_0)},\,\frac{1+\delta}{G(\lambda_0)}\right]\right)
\;\le\; 2(M/|\mu_0|)^{-\vartheta} + 2\exp\left(-\gamma \frac{n}{\eta M^2}\right).
\]

\end{theorem}

Since the proof uses ideas which are developed in the course of the large deviations theory for the deflationary regime, it is deferred until Appendix \ref{appendix: mod inf}. \begin{remark}
	The second term arises from considering the probability that the kernel $\lambda(t)$ changes enough to destroy the approximation \eqref{eq: find the RW approximately} before $\TPM$. In order for the corresponding error term to have a similar rate of decay to the first, we must restrict to $$ M\lesssim \sqrt{\frac{n}{\eta \log(n/\eta)}}$$ \findme consistently with \eqref{eq: restrict to scales}. \end{remark}

\section{Moderate Spikes in the Deflationary Regime}\label{sec: LDP}

We now turn to the deflationary regime, where the log-drift in the lazy training regime is negative. We first introduce relevant notation.
Assuming \eqref{eq: minibatch critical values}, we define $\vartheta(\lambda)=\vartheta(\lambda, \eta)$ as the maximal root to
\begin{equation}\label{eq: eq defining alpha}
\theta\mapsto\sum_{i=1}^m p_i\,\big|1-\eta \lambda s_i^2\big|^{\theta}=1.
\end{equation}
The focus of this section is the case $G(\lambda)<0$, in which case the derivative of the function on the left-hand side is negative at $\theta=0$; the function is always convex, and assuming \eqref{eq: minibatch critical values}, diverges to infinity as $\theta\to \infty$, which together guarantees a unique positive root $\vartheta(\lambda)>0$. This root has the property that it makes the approximate model, ignoring higher-order terms as in \eqref{eq: find the RW approximately} $$ Z_t:= (|1-\eta \lambda_0s_{i(t)}^2|\dots |1-\eta \lambda_0s_{i(1)}^2|\mu_0)^{\vartheta(\lambda)} $$ into a non-negative martingale, see Step 1 of Lemma \ref{lemma: new hitting negative}. This allows a derivation of \begin{itemize} \item a large-deviations upper bound on hitting probabilities using optional stopping (Step 2); \item a matching lower bound via a `tilting' argument, using $Z_t$ to define a change of probability measure $\QQ$ under which the log-drift is positive (Steps 3-5).\end{itemize} For an initial prediction $\mu_0$ and target $M$, both upper and lower bounds derived in this way decay like $(|\mu_0|/M)^{\vartheta(\lambda)}$. Lemma \ref{lemma: new hitting negative} presents these arguments, keeping track also of the neglected higher-order terms. The remainder of the section is dedicated to showing that the bounds are not affected by the updates to the kernel, provided that $M$ is chosen according to \eqref{eq: restrict to scales} for some suitably chosen $\beta>0$.  Motivated by these bounds, the theorem which results from this discussion is as follows.

\begin{theorem} \label{thm: LDP spikes}
Under Assumption \ref{hyp: assumption linear model}, assume that $G(\lambda_0)<0$  and that \eqref{eq: minibatch critical values} holds. Then, there exists a constant $C\in(0,\infty)$ depending only on the data $(p_i,s_i)$, $\eta$, and the intervals below, such that the following hold.  Let $\vartheta>0$ be given by \eqref{eq: eq defining alpha} and let $0<\vartheta_-<\vartheta<\vartheta_+$. Then, for some $c=c(\{p_i, s_i\}, \eta, \vartheta_\pm)$ and some $\beta>0$, and all thresholds
\[
|\mu_0|\ \le\ M\ \le\ \frac{c}{\log^{1/\beta}(n/\eta)}\,\sqrt{\frac{n}{\eta}}
\]
we have the two-sided decay rate
\begin{equation}\label{eq: perturbative LDP spike}
C^{-1}\Big(\tfrac{M}{|\mu_0|}\Big)^{-\vartheta_+}
\ \le\
\mathbb{P}_{\mu_0, \lambda_0}\big(\TPM<\infty\big)
\ \le\
C\Big(\tfrac{M}{|\mu_0|}\Big)^{-\vartheta_-}.
\end{equation} 

\end{theorem}

We begin with a lemma, which deals with the parts of the proof which can be done perturbatively.  \begin{lemma}[Hitting Probabilities, Deflationary Case]\label{lemma: new hitting negative} Assume, in the notation of Assumption \ref{hyp: assumption linear model}, that $G(\lambda_0)<0$, and that $\lambda_0\eta s_\star^2>2$. For $\epsilon, \kappa>0$, let $\vartheta_{\mp}(\epsilon, \kappa)$ be given by \eqref{eq: eq defining alpha} for the shifted indexes $b_i^\pm$ given by (\ref{eq: def bi+} - \ref{eq: def bi-}): \begin{equation}
		\label{eq: choose e k from vartheta new} \sum_{i=1}^m p_i (b_i^+(\epsilon, \kappa, \lambda))^{\vartheta_-}=1; \qquad \sum_{i=1}^m p_i (b_i^-(\epsilon, \kappa, \lambda))^{\vartheta_+}=1
	\end{equation} and assume that $\epsilon, \kappa$ are small enough that \begin{equation}
		\label{eq: mild epsilon kappa} \max_i b^+_i(\epsilon,\kappa, \lambda_0)<4.
	\end{equation} Then, for all $|\mu_0|\le M\le \kappa\sqrt{n/\eta}$, \begin{enumerate}[label=\roman*).]
		\item (Upper Bound)  It holds that \begin{equation}
			\label{eq: main hitting probability conclusion UB} \PP_{\mu_0, \lambda_0}(\TPM<\TK{\lambda_0-\epsilon}, \TPM<\infty) \le \left(\frac{|\mu_0|}{M}\right)^{\vartheta_-}.
		\end{equation} \item (Lower Bound, with hitting time) There exists $v=v(\lambda_0)>0$ such that, for all $k\in[1,\infty)$ such that \begin{equation}
			\label{eq: consistency condition hitting LB} M^2\log(4M/|\mu_0|) < \frac{\epsilon n}{\eta a_+ vk}
		\end{equation} it holds that \begin{equation}\label{eq: main hitting probability conclusion LB} \EE_{\mu_0, \lambda_0}\left[\left(\frac{|\mu(\TPM)|}{M}\right)^{\vartheta_+}\indiq(\TPM\le vk\log (4M/|\mu_0|)\right] \ge \left(1-\frac1{2k}\right) \left(\frac{|\mu_0|}{M}\right)^{\vartheta_+}.\end{equation} 
\end{enumerate}
	 \end{lemma}\begin{proof}
	 	We deal with the two cases separately. \step{1. Identification of Super- and Submartingales} We first use the relation \eqref{eq: choose e k from vartheta new} to show that, with stopping at time $$T:=\tau^{\rm Pred}_{\kappa\sqrt{n/\eta}}\land \TK{\lambda_0-\epsilon},$$ the process $|\mu(t\land T)|^{\vartheta_-}$ is a (non-negative) supermartingale. The same argument applies, {\em mutatis mutandis}, to show that $|\mu(t\land T)|^{\vartheta_+}$ is a submartingale. \medskip \\  Fix $t\ge 0$. On the event $\{T\le t\}$, we have $\mu((t+1)\land T)=\mu(t\land T)$, producing \begin{equation}\label{eq: submartingale case 1} \mathbb{E}_{\mu_0, \lambda_0}[|\mu((t+1)\land T)|^{\vartheta_-}|\mathcal{F}(t)]\indiq(T\le t)= |\mu(t\land T)|^{\vartheta_-}\indiq(T\le t). \end{equation} Meanwhile, on the event $\{T>t\}$, the exact update rule \eqref{eq: SGD 2} shows that \begin{equation}
	 		\label{eq: one step of coupling} |\mu((t+1)\land T)|^{\vartheta_-}=|\mu(t\land T)|^{\vartheta_-}\left|1-\eta \lambda(t)s^2_{i(t+1)}+\frac{\eta^2 s_{i(t+1)}^4 \mu(t)^2}{n}\right|^{\vartheta_-}
	 	\end{equation} for the index $i(t+1)$, which is sampled independently of $\mathcal{F}(t)$. By the Definitions \eqref{eq: def bi}, \eqref{eq: def bi+}, and because $t<T$ by assumption, the absolute value on the right-hand side is at most $b^+_{i(t+1)}$, and taking out what is known, \begin{equation}
	 		\label{eq: submartingale case 2}  \mathbb{E}_{\mu_0, \lambda_0}[|\mu((t+1)\land T)|^{\vartheta_-}|\mathcal{F}(t)]\indiq(T> t)\le |\mu(t\land T)|^{\vartheta_-}\indiq(T>t)\mathbb{E}\left[(b^+_{i(t+1)})^{\vartheta_-}|\mathcal{F}(t)\right].
	 	\end{equation} The final conditional expectation is the same as the (unconditional) expectation $\sum_i p_i (b_i^+)^{\vartheta_-}=1$, because $i(t+1)$ is independent of $\mathcal{F}(t)$ and using the Definition \eqref{eq: choose e k from vartheta new}. Gathering the two cases (\ref{eq: submartingale case 1}, \ref{eq: submartingale case 2}), we conclude that $|\mu(t\land T)|^{\vartheta_-}$ is a supermartingale, as claimed. \step{2. Upper Bound} We next use the supermartingale identified in Step 1 to prove \eqref{eq: main hitting probability conclusion UB}. Writing $Z^-(t)$ for the supermartingale constructed in Step 1, we observe that $Z^-(t\land \TPM)\le (M\max_i b^+_i)^{\vartheta_-} \le (4M)^{\vartheta_-}$ is uniformly bounded thanks to \eqref{eq: mild epsilon kappa}. Since we chose $M\le \kappa \sqrt{n/\eta}$, the only way the stopping $T$ can occur before $\TPM$ is if $\{\TK{\lambda_0-\epsilon}<\TPM\}$. Consequently, on the event $\{\TPM<\TK{\lambda_0-\epsilon}\}$, we have the lower bound $$Z^-(\TPM)\ge M^{\vartheta_-}.$$ The optional stopping theorem gives \begin{equation}
	 		\label{eq: opt stop UB} |\mu_0|^{\vartheta_-}=\EE_{\mu_0, \lambda_0}[Z^-(0)] \ge \EE_{\mu_0, \lambda_0}[Z^-(\TPM)] \ge M^{\vartheta_-}\PP_{\mu_0,\lambda_0}(\TPM<\TK{\lambda_0-\epsilon})
	 	\end{equation} which rearranges to the claim \eqref{eq: main hitting probability conclusion UB}. \step{3. Lower Bound: Change of Measure} We divide the proof of  \eqref{eq: main hitting probability conclusion LB} into further steps. For the first step, we construct a `tilting' (the Cram\'er transform) and define a probability measure $\mathbb{Q}$ which is absolutely continuous with respect to $\PP_{\mu_0, \lambda_0}$. In subsequent steps, we will see that the construction makes the event $\{\TPM<vk\log(4M/|\mu_0|)\}$ an event of high $\QQ$-probability, provided that $v$ is chosen suitably, which will ultimately lead to the desired estimate \eqref{eq: main hitting probability conclusion LB}. We refer to \cite[Proof of Theorem 2.2.3, p.32]{dembo2009large} for an exposition of the general strategy based on the Cram\'er transform.\\\\ Writing $Z^+(t)$ for the submartingale constructed in Step 1, we define, for $t\ge 0$, random variables\begin{equation}
	 		\rho(t):=\mathbb{E}_{\mu_0,\lambda_0}\left[\frac{Z^+(t+1)}{Z^+(t)}|\cF(t)\right].
	 	\end{equation} From Step 1, it follows that $\rho(t)\ge 1$ almost surely, and the process $$ Z(t):=\frac{Z^+(t)}{|\mu_0|^{\vartheta_+}\rho(0)\dots \rho(t-1)} $$ is a mean-1, non-negative martingale. Reasoning as in Step 2 above, the stopped martingale $Z^{\TPM}$ is additionally uniformly bounded. We may therefore define a change of measure $\QQ\ll\PP_{\mu_0,\lambda_0}$ by setting \begin{equation}
	 		\frac{d\QQ}{d\PP_{\mu_0, \lambda_0}}:=Z(\TPM).
	 	\end{equation} \step{4. Lower Bound: Uniformly Positive Drift under $\QQ$} In order to show that the events appearing in \eqref{eq: main hitting probability conclusion LB} becomes typical under $\QQ$, for $v$ to be chosen, we now consider the $\log$-drift, under this new measure, for $|\mu(t)|$. In this step, we identify the sampling probabilities $p^{\QQ,t}_i$ under the change of measure. This will allow us to choose some $v=v(\lambda_0)<\infty$ such that, up to time $\TPM\land T$, the $\log$-drift is uniformly bounded below by $\frac1{2v}$. This will allow us, in the final step, to show that the hitting time $\TPM$ occurs with high-probability on timescales $\sim v\log(4M/|\mu_0|)$. \\ \\ For $v$ to be chosen, we assume that $k$ satisfies \eqref{eq: consistency condition hitting LB} and consider $t< vk\log (M/|\mu_0|) \land \TPM$. By the assumption on $k$ and using the explicit change \eqref{eq: SGD 2}, the total change in the kernel up to time $t$ is at most $a_+ M^2 t \eta/n <\epsilon$, so $$ \{t< vk\log (M/|\mu_0|) \land \TPM\}\subset \{t<T\land \TPM\}. $$ Fixing $A\in \cF(t)$ and $1\le i\le m$, \begin{equation}
	 		\begin{split}
	 			\QQ(A,i(t+1)=i, t< \TPM) &=\EE_{{\mu_0,\lambda_0}}\left[Z(\TPM)\indiq(A, i(t+1)=i,t< \TPM)\right] \\ & =\EE_{{\mu_0,\lambda_0}}\left[Z(t+1)\indiq(A, i(t+1)=i,t< \TPM)\right]
	 		\end{split}
	 	\end{equation} where, in the second line, we use that $Z^{\TPM}$ is a martingale. Using the definition, we get \begin{equation}\begin{split} 
	 		\dots &=\EE_{{\mu_0, \lambda_0}}\left[Z(t)\indiq(A, t<\TPM)\EE_{{\mu_0, \lambda_0}}\left[\frac{Z^+(t+1) \indiq(i(t+1)=i)}{Z^+(t)\rho(t)}|\cF(t)\right]\right] \\ & \hspace{3cm}=\EE_{\QQ}\left[\indiq(A, t<\TPM)p^{\QQ,t}_i\right]  \end{split}	\end{equation} where $p^{\QQ,t}_i$ are the tilted probabilities, given while $t<\TPM$ by \begin{equation}\label{eq: def pqi} p^{\QQ,t}_i=\frac{1}{\rho(t)}\left|1-\eta \lambda(t)s_i^2 + \frac{\eta^2 s_i^4 \mu(t)^2}{n}\right|^{\vartheta_+}p_i\end{equation} which determines $\rho(t)$ as the normalising factor. Consequently, the conditional mean is bounded below by \begin{equation}\begin{split}
	 			& \EE_{\QQ}\left[\log \left(\frac{|\mu(t+1)|}{|\mu(t)|}\right)|\cF(t), t<\TPM\right] \\ & \hs \hs=\frac1{\rho(t)}\sum_{i=1}^m p_i\left|1-\eta \lambda(t)s_i^2 + \frac{\eta^2 s_i^4 \mu(t)^2}{n}\right|^{\vartheta_+} \log \left|1-\eta \lambda(t)s_i^2 + \frac{\eta^2 s_i^4 \mu(t)^2}{n}\right| \\[1ex] &\hs\hs =\left.\frac{d}{d\theta}\right|_{\theta=\vartheta_+}\underbrace{\log \sum_{i=1}^m p_i \left|1-\eta \lambda(t)s_i^2 + \frac{\eta^2 s_i^4 \mu(t)^2}{n}\right|^\theta}_{=:\psi_t(\theta)}. \end{split}
	 		\end{equation} The cumulant generating function on the right-hand side is a convex function of $\theta$, and so the gradient may be estimated from below by $(\psi_t(\vartheta_+)-\psi_t(\varphi))/(\vartheta_+-\varphi)$, for any $\varphi<\vartheta_+$. Choosing $\varphi=\vartheta_-/2$, and noticing that $\psi_t(\vartheta_+)\ge 0$, we obtain an almost sure and uniform lower bound \begin{equation}\begin{split} \label{eq: drift LB}
	 			& \EE_{\QQ}\left[\log \left(\frac{|\mu(t+1)|}{|\mu(t)|}\right)|\cF(t), t<\TPM\right] \ge -\frac{2}{2\vartheta_+-\vartheta_-}\log \sum_{i=1}^m (b^+_i)^{\vartheta_-/2}=:\frac{1}{2v} \end{split} \end{equation} where the final equality defines $v>0$.  \step{5. Lower Bound: Reassembly}  We are now able to assemble the previous two steps, together with a further martingale calculation, to prove the claim \eqref{eq: main hitting probability conclusion LB} for $v$ constructed in Step 4. We first use a martingale calculation, based on \eqref{eq: drift LB} to show that the event appearing in \eqref{eq: main hitting probability conclusion LB} has $\QQ$-probability at least $1-\frac1{2k}$, from which we deduce the claim \eqref{eq: main hitting probability conclusion LB} by using the definition of $\QQ$ in Step 3.\\\\ From Step 4, it follows that $S^{\TPM\land vk\log (4M/|\mu_0|)}$ is a $\QQ$-submartingale, where we define \begin{equation} S(t):=\log \frac{|\mu(t)|}{|\mu_0|}-\frac{t}{2v}. \label{eq: S Q submg} \end{equation} From \eqref{eq: mild epsilon kappa}, it follows that $S^{\TPM}$ is bounded above by $\log(4M/|\mu_0|)$ and we may use the optional stopping theorem to find \begin{equation}\begin{split}
	 				\label{eq: expected hitting time} &0=\mathbb{E}_{\QQ}S(0) \le \mathbb{E}_{\QQ}S(\TPM\land vk\log (4M/|\mu_0|)) \\& \hs \hs \le \log(4M/|\mu_0|)- \frac1{2v} \mathbb{E}_{\QQ}[\TPM\land vk\log (4M/|\mu_0|)].
	 			\end{split}\end{equation} It follows from a Chebyshev inequality that\begin{equation}
	 				\label{eq: Q prob} \QQ(\TPM\le  vk\log(4M/|\mu_0|)) \ge 1- \frac1{2k}.
	 			\end{equation} To conclude, we rewrite the expectation in \eqref{eq: main hitting probability conclusion LB} as \begin{equation}
	 				\begin{split}&\EE_{\mu_0, \lambda_0}\left[\left(\frac{|\mu(\TPM)|}{M}\right)^{\vartheta_+}\indiq(\TPM\le vk\log (4M/|\mu_0|))\right] \\ & \hs =\left(\frac{|\mu_0|}{M}\right)^{\vartheta_+}\mathbb{E}_{\QQ}\left[\rho(0)\rho(1)\dots \rho(\TPM-1)\indiq(\TPM\le vk\log (4M/|\mu_0|))\right] \\ & \hs \ge \left(\frac{|\mu_0|}{M}\right)^{\vartheta_+}\QQ(\TPM\le vk\log (4M/|\mu_0|)) \end{split} 
	 			\end{equation} where the first line uses the definitions of $Z$ and $\QQ$, and in the second line we recall that $\rho(t)\ge 1$ almost surely. The conclusion \eqref{eq: main hitting probability conclusion LB} now follows from \eqref{eq: Q prob}. \end{proof}

Before turning to the proof of Theorem \ref{thm: LDP spikes}, we discuss and prove an additional proposition. In the upper bound of \eqref{eq: perturbative LDP spike}, we seek to understand the probability that the prediction error $|\mu(t)|$ {\em ever} reaches the threshold, even on timescales much larger than $\log n$. As discussed above \eqref{eq: restrict to scales}, it is in principle possible for a kernel decrease to occur first via a `slow escape', while \eqref{eq: main hitting probability conclusion UB} only allows shows bounds of the desired form if we assume that $\TPM<\TK{\lambda_0-\epsilon}$. To bridge this gap, we give the following proposition, which justifies the claim in the introduction that large spikes are the most probable way for the dynamics to escape the lazy training regime on any timescale. \begin{proposition}[Slow Escape is Exponentially Improbable] \label{prop: slow escape} Under the same hypotheses and with the same notation as Theorem \ref{thm: LDP spikes}, let $\epsilon, \kappa>0$ be such that the exponent $\vartheta_-'$ given by \eqref{eq: eq defining alpha} for the shifted indexes $b^+_i(\lambda_0, \epsilon,\kappa)$ is positive. Then, for any $M\le \min(1,\kappa)\sqrt{n/\eta}$ and $\beta\in (0,\vartheta'_-)\cap(0,2]$, there exists $\gamma=\gamma(\beta)>0$ such that \begin{equation} \label{eq: slow escape conclusion 1}
		\mathbb{P}_{\mu_0,\lambda_0}(\TK{\lambda_0-\epsilon}\le \TPM)\le 2\exp\left(-\epsilon \gamma \left(\frac{n}{\eta M^2}\right)^{\beta/2}\right).\end{equation}	
	
\end{proposition}

\begin{proof} For $M, \beta$ as given, let $T$ denote the stopping time $T:=\TK{\lambda_0-\epsilon}\land \TPM$. We first find a polynomial moment of $|\mu(t)|$, summed up to time $T$, which is then upgraded to an exponential moment by the Khas'minskii lemma. \step{1. Polynomial Moment} First, consider the process $Z_t:=|\mu(t\land T)|^{\beta}$. Arguing as in Step 1 of Lemma \ref{lemma: new hitting negative} and using that $\beta<\vartheta'_-$, for any $|\mu|\in [|\mu_0|, M], \lambda\in [\lambda_0-\epsilon, \lambda_0]$, \begin{equation}
	\mathbb{E}_{\mu, \lambda}[Z_{t+1}|\cF_t] \le c_\beta |\mu(t)|^{\beta}\indiq(T>t)+Z_t\indiq(T\le t)
\end{equation} for $c_{\beta}:=\sum p_i (b^+_i)^{\beta}<1$. Defining \begin{equation}
	M_t:=Z_t+ (1-c_\beta)\sum_{s<t\land T}|\mu(s)|^{\beta}
\end{equation} we obtain \begin{equation} \label{eq: M is for martingale} \begin{split}\EE_{\mu, \lambda}[M_{t+1}-M_t|\cF_t] & = \EE_{\mu, \lambda}[Z_{t+1}|\cF_t]-|\mu(t\land T)|^\beta +(1-c_\beta)|\mu(t)|^\beta \indiq(T>t) \\ & \le 0\end{split} \end{equation} by considering the events $\{T\le t\}, \{T>t\}$ separately. $M_t$ is thus a non-negative supermartingale, and optional stopping produces, for the same range of $\mu, \lambda$,  \begin{equation} \label{eq: infinite time moment est}
	\mathbb{E}_{\mu, \lambda}\sum_{s<T}|\mu(s)|^{\beta} \le \frac{|\mu_0|^{\beta}}{1-c_\beta}.
\end{equation}  \step{2. Exponential Moment by Khas'minskii's Lemma} We next upgrade the estimate (\ref{eq: infinite time moment est}) to exponential moments. The time $T$ is the first time that the Markov chain $(\mu_t, \lambda_t)$ leaves the rectangle $[-M,M]\times (\lambda_0-\epsilon, \lambda_0)$, and, uniformly over all initial conditions in this rectangle, \eqref{eq: infinite time moment est} produces \begin{equation}
	\frac{1-c_\beta}{2M^\beta}\mathbb{E}_{\mu, \lambda}\sum_{s<T}|\mu(s)|^\beta \le \frac12.
\end{equation} This is the setting of Khas'minskii's lemma (see, for example, \cite{stummer2000exponentials} and references therein) which upgrades the expectation to the exponential moment \begin{equation} \label{eq: khasminskii unified}
	\mathbb{E}_{\mu_0, \lambda_0}\left[\exp\left(\frac{1-c_\beta}{2M^\beta}\sum_{s<T}|\mu(s)|^\beta\right)\right] \le 2. 
\end{equation} \step{3. Markov Inequality} We now identify events to apply Markov's inequality. By considering the maximal kernel decrease at each step, we have \begin{equation}\label{eq: contain event slow escape}\{\TK{\lambda_0-\epsilon}<\TPM\}\subset \left\{ \frac{\eta}{n} \sum_{s<T}  |\mu(s)|^2 \ge \frac{\epsilon}{a_+}\right\} \end{equation}  where $a_+$ is given by \eqref{eq: decrease in curvature} and depends only on $\lambda_0, \{s_i\}$. Since $\eta |\mu(s)|^2/n \le 1$ for every $s<T$, and $\beta\le 2$, we get \begin{equation}
	\left\{\TK{\lambda_0-\epsilon}<\TPM\right\}\subset \left\{\sum_{s<T} \left(\frac{\eta}{n}\right)^{\beta/2}|\mu(s)|^{\beta} \ge \frac{\epsilon}{a_+}\right\} 
\end{equation} which, by Markov's inequality and \eqref{eq: khasminskii unified}, has probability at most \begin{equation}
	\PP_{\mu_0, \lambda_0}(\TK{\lambda_0-\epsilon}<\TPM)\le  2\exp\left(-\frac{\epsilon (1-c_\beta)}{2a_+}\left(\frac{n}{\eta M^2}\right)^{\beta/2}\right).
\end{equation} which is a bound of the desired form, taking $\gamma(\beta):=\frac{1-c_\beta}{2a_+}$.	\end{proof}

 We are now able to give the
\begin{proof}[Proof of Theorem \ref{thm: LDP spikes}]
	We divide the proof to deal with the two cases. \step{1. Coupling \& Parameter Choice} For $\vartheta_\pm$ as given, we observe that the map $\{b_i, p_i\}\mapsto \vartheta$ defined implicitly by \eqref{eq: eq defining alpha} is continuous and monotone in the bases $b_i$. We may therefore choose $\epsilon, \kappa>0$, depending on the data and $\vartheta_\pm$, and satisfying \eqref{eq: mild epsilon kappa} such that the bases $b^{\pm}_i$ given in (\ref{eq: def bi+} - \ref{eq: def bi-}) satisfy \begin{equation}
		\label{eq: choose e k from vartheta} \sum_{i=1}^m p_i (b_i^+(\epsilon, \kappa, \lambda))^{\vartheta_-}<1; \qquad \sum_{i=1}^m p_i (b_i^-(\epsilon, \kappa, \lambda))^{\vartheta_+}>1
	\end{equation} from which it follows that $\vartheta_-':=\vartheta(\{b_i^+, p_i\}), \vartheta'_+:=\vartheta(\{b_i^-, p_i\})$ satisfy \begin{equation}
		\label{eq: compare various exponents} 0<\vartheta_-<\vartheta'_-<\vartheta<\vartheta'_+<\vartheta_+.
	\end{equation}  We recall the definition of $a$ from \eqref{eq: decrease in curvature}, let $v$ be as in Lemma \ref{lemma: new hitting negative}ii), let $\beta$ be any value allowed by Proposition \ref{prop: slow escape}, and let $\gamma=\gamma(\beta)$ be the corresponding value produced by the proposition. With these fixed, we choose \begin{equation} \label{eq: choice of c LDP theorem} c:=\min\left(1, \sqrt{\frac{\epsilon}{a_+v}},\kappa, \left(\frac{\epsilon \gamma}{\vartheta_-}\right)^{1/\beta}\right) \end{equation} and restrict, throughout, to $n, \eta$ satisfying \begin{equation}
		\frac{n}{\eta}\ge e;\qquad \frac12\left(\frac{n}{\eta}\right)^{\beta/2}\ge 2.
	\end{equation} 
	
	 \step{2. Scale Decomposition \& Setup} In order to set up the proofs of both the upper and the lower bound, we now set up a scale decomposition. This allows to identify, in both proofs, a dominant term which can be bounded by Lemma \ref{lemma: new hitting negative}, and error terms which must be estimated. As in Step 1, we have $\beta>0$ and associated constant $\gamma=\gamma_\beta>0$ given by Proposition \ref{prop: slow escape}. We now define \begin{equation}
		\label{eq: def Li} L_j:=\left(\frac{n}{\eta}\right)^{-j/2}M; \qquad 0\le j<k:=\left\lceil\frac{2\log(M/|\mu_0|)}{\log(n/\eta)}\right\rceil; 
	\end{equation} \begin{equation}
		L_k:=|\mu_0|
	\end{equation} which produces a strictly decreasing scale decomposition $M=L_0>L_1>\dots >L_k=|\mu_0|$, with additionally $L_{j+1}/L_j\ge (n/\eta)^{-1/2}$. For $j=0,\dots,k-1$, we define a decreasing sequence of stopping times by \begin{equation}
		T_j:=\tau^{\rm Pred}_{L_j}\land \TK{\lambda_0-\epsilon/ 2^{j}}.
	\end{equation} We treat the event $\{\TK{\lambda_0-\epsilon}<\TPM\}$ as the error event where the approximation via geometric random walk breaks down, and the events $\{T_j=\TK{\lambda_0-\epsilon/2^j}<\infty\}$ as a scale-wise decomposition of this event. The main work of the proof rests on showing that they have probabilities much smaller than the events $\{T_j=\tau^{\rm Pred}_{L_j}<\infty\}$. For the lower bound, we use the containment \begin{equation}
		\label{eq: tree decomposition LB} \{\TPM<\infty, \TPM<\TK{\lambda_0-\epsilon}\}\supset \bigcap_{j=0}^{k-1} \underbrace{\{T_j=\tau^{\rm Pred}_{L_j}<\infty, T_j<\TK{\lambda_0-\epsilon/2^j}\}}_{:=A^-_j}.
	\end{equation} For the upper bound, on the event $\{\TK{\lambda_0-\epsilon}<\TPM\}$, there exists a unique maximal index $J$ with $T_J=\TK{\lambda_0-\epsilon/2^J}$, producing the containment \begin{equation}
		\label{eq: tree decomposition UB} \{\TK{\lambda_0-\epsilon}<\TPM\}\subset \bigcup_{j=0}^{k-1} \underbrace{\{T_\ell=\tau^{\rm Pred}_{L_\ell}\text{ for all }j<\ell<k, T_j=\TK{\lambda_0-\epsilon/2^j}\}}_{:=A^+_j}.
	\end{equation} 
	
	\step{3. Lower Bound} For the lower bound, we use the coupling and Lemma \ref{lemma: new hitting negative} to estimate the probability of intersections of events $A^-_j$ in \eqref{eq: tree decomposition LB} below. We prove, by induction on $0\le j\le k$, \begin{equation} \begin{split}
		\label{eq: LB inductive claim} & \inf\left\{\left(\frac{\mu'}{L_j}\right)^{-\vartheta'_+}\mathbb{P}_{\mu', \lambda'}\left(\bigcap_{0\le \ell<j} A^-_\ell\right): \mu'\in [L_j, 4L_j], \lambda'\in \left[\lambda_0-\frac{\epsilon}{2^{j+1}},\lambda_0\right] \right\} \\ & \hs \ge \left(\frac{L_j}{4M}\right)^{\vartheta'_+}\prod_{0\le \ell <j}\left(1-\frac12\left(\frac{n}{2\eta}\right)^{-\ell}\right). \end{split}
	\end{equation} Once this is established, the special case $j=k, \mu'=\mu_0=L_k$ produces the lower bound asserted in the theorem. \\ \\ We now prove the inductive assertion. The base case $j=0$ is trivial, since the intersection is empty. Given case $j<k$, fix $\mu'\in [L_{j+1}, 4L_{j+1}], \lambda'\in [\lambda_0-\epsilon/2^{j+2},\lambda_0]$, and consider the dynamics started at $\mu', \lambda'$. By the strong Markov property, on the event $A^-_j$, it holds that $\tau^{\rm Pred}_{L_j} \in [L_j, 4L_j]$ thanks to \eqref{eq: mild epsilon kappa}, and $\lambda(T_j)\ge \lambda_0-\epsilon/2^j$. We therefore get \begin{equation}
		\label{eq: inductive LB}  \begin{split} \mathbb{P}_{\mu', \lambda'}\left(\bigcap_{0\le \ell\le j} A^-_\ell \right) & \ge \EE_{\mu',\lambda'}\left[\indiq(A^-_j)\mathbb{P}_{\mu(T_j),\lambda(T_j)}\left(\bigcap_{0\le \ell<j}A^-_\ell \hspace{0.1cm} \bigg|\hspace{0.1cm} A^-_j \right)\right] \\ & \ge \left(\frac{L_j}{4M}\right)^{\vartheta'_+} \left(\prod_{0\le \ell<j}\left(1-\frac12\left(\frac{n}{2\eta}\right)^{-\ell}\right)\right)\EE_{\mu', \lambda'}\left[\indiq(A^-_j)\left(\frac{\mu(T_j)}{L_j}\right)^{\vartheta'_+}\right] \end{split}
	\end{equation} where the second line follows using the inductive hypothesis. Since $\lambda'\ge \lambda_0-\frac{\epsilon}{2^{j+1}}$, for the given starting point we have $$ A_j^-\supset \left\{T_j \le \left(\frac{n}{\eta}\right)^j\log\left(\frac{n}{\eta}\right)\frac{\epsilon}{2^{j+1}c^2 a_+}\right\} \supset \left\{T_j \le \left(\frac{n}{2\eta}\right)^jv\log\left(\frac{L_j}{\mu'}\right)\right\}$$ thanks to the choice \eqref{eq: choice of c LDP theorem}, and where we recall $v$ is given by Lemma \ref{lemma: new hitting negative}. Thanks to \eqref{eq: main hitting probability conclusion LB}, we thus find \begin{equation} \begin{split}\label{eq: chain inductive step} \EE_{\mu', \lambda'}\left[\indiq(A^-_j)\left(\frac{\mu(T_j)}{L_j}\right)^{\vartheta'_+}\right] & \ge \EE_{\mu', \lambda'}\left[\indiq\left\{\tau^{\rm Pred}_{L_j}\le \left(\frac{n}{2\eta}\right)^jv\log\left(\frac{L_j}{\mu'}\right)\right\}\left(\frac{\mu(\tau^{\rm Pred}_{L_j})}{L_j}\right)^{\vartheta'_+}\right] \\ & \ge \left(1-\frac12\left(\frac{n}{2\eta}\right)^{-j}\right)\left(\frac{\mu'}{L_j}\right)^{\vartheta'_+}.\end{split}\end{equation} Substituting this bound back into \eqref{eq: inductive LB} rearranges to the inductive claim \eqref{eq: LB inductive claim}, now proven for case $j+1$.    \step{4. Upper Bound}  We start from the decomposition   \begin{equation} \label{eq: UB LDP decomposition}
		\{\TPM<\infty\}\subset \{\TPM<\TK{\lambda_0-\epsilon}, \TPM<\infty\}\cup\{\TK{\lambda_0-\epsilon}<\TPM\}.
	\end{equation} The first term can be handled immediately by Lemma \ref{lemma: new hitting negative}, and is bounded above by $(|\mu_0|/M)^{\vartheta_-}$ thanks to \eqref{eq: compare various exponents}. In order to complete the proof of the upper bound, we now show a bound of the same order $(|\mu_0|/M)^{\vartheta_-}$ for the probability of the error event in \eqref{eq: UB LDP decomposition}, which we recall from \eqref{eq: tree decomposition UB} is partitioned by the error events $A^+_j, 0\le j\le k-1$. We estimate \begin{equation} \label{eq: SMP}
		\PP_{\mu_0, \lambda_0}(A_j)\le \mathbb{P}_{\mu_0, \lambda_0}(\tau^{\rm Pred}_{L_{j+1}}<\TK{\lambda_0-\epsilon/2^{j+1}})\mathbb{P} 
		_{\mu_0, \lambda_0}(A_j|\tau^{\rm Pred}_{L_{j+1}}<\TK{\lambda_0-\epsilon/2^{j+1}}).
	\end{equation} For $j\le k-1$, Lemma \ref{lemma: new hitting negative} bounds the probability of the conditioning event above by $(|\mu_0|/L_{j+1})^{\vartheta'_-}$.  The remaining conditional probability is estimated using the strong Markov property: The process $$(\hat{\mu}(t), \hat{\lambda}(t))_{t\ge 0}:=(\mu(T_j+t), \lambda(T_j+t))_{t\ge 0}$$ has, on the event $\{T_j<\infty\}$, the same dynamics as the original process, and $A_j$ requires that the restarted chain achieve a curvature reduction of $\epsilon/2^{j+1}$ while keeping the prediction below $L_j$. This is the setting of Proposition \ref{prop: slow escape}, and we obtain \begin{equation}
		\PP_{\mu_0, \lambda_0}(A_j|\tau^{\rm Pred}_{L_{j+1}}<\TK{\lambda_0-\epsilon/2^{j+1}})\le 2\exp\left(-\frac{\epsilon \gamma}{2^{j+1}}\left(\frac{n}{\eta L_j^2}\right)^{\beta/2}\right)=2\exp\left(-\frac{\epsilon \gamma}{2^{j+1} M^\beta}\left(\frac{n}{\eta}\right)^{\beta (j+1)/2}\right).
	\end{equation} By the choice of $M$ in \eqref{eq: choice of c LDP theorem}, we may simplify the exponent to \begin{equation}
		\dots \le 2\exp\left(-\frac{\vartheta_-}{2}(\log (n/\eta))\left(\frac12\left(\frac{n}{\eta}\right)^{\beta/2}\right)^{j}\right).
	\end{equation} \findme As above, we restrict to $n, \eta$ with $(n/\eta)^{\beta/2}>4$, for which we can estimate the exponential in brackets as $2^j\ge 2j$, to find overall, for $j\ge 1$, \begin{equation}
		\label{eq: penultimate bound Ak} \mathbb{P}_{\mu_0, \lambda_0}(A_j|\tau^{\rm Pred}_{L_{j+1}}<\TK{\lambda_0-\epsilon/2^{j+1}}) \le 2\left(\frac{n}{\eta}\right)^{-\vartheta'_-j} \le 2\left(\frac{L_{j+1}}{M}\right)^{\vartheta'_-}\left(\frac{n}{\eta}\right)^{-\vartheta'_-(j-1)/2}.	\end{equation} The same holds for $j=0$ with the final factor replaced by 1. Recombining using \eqref{eq: SMP} and summing over all $k$, and recalling that we assume $n/\eta\ge 2$, we get \begin{equation}
			\label{eq: bad event} \mathbb{P}_{\mu_0, \lambda_0}(\TK{\lambda_0-\epsilon}<\TPM)\le 2\left(\frac{|\mu_0|}{M}\right)^{\vartheta_-}\left(1+\sum_{k\ge 0}2^{-\vartheta'_- j/2}\right)
		\end{equation} and since the final sum is convergent, we may absorb it into the prefactor. Combining with the conclusion of Step 3 and returning to \eqref{eq: UB LDP decomposition}, we have the upper bound asserted in \eqref{eq: perturbative LDP spike}.\end{proof}

\section{Large Spikes}  \label{sec: large LLN} We now turn to the question of how large spikes can become {\em after} they have reached the scale $$M\sim \sqrt{n/(\eta \log(n/\eta))}$$ \findme allowed by Theorems \ref{thm:lln moderate} - \ref{thm: LDP spikes}. In this case, the kernel $\lambda(t)$ can no longer be treated as approximately constant, since it decreases by $\mathcal{O}(1)$ at each time step. The goal is to bound the typical reduction of the curvature due to the spike below, and to show that it happens on a timescale much shorter than the spikes. We begin with the inflationary regime and consider $\lambda_0$ such that $G(\lambda_0)>0$, and define  \begin{equation}
	\label{eq: def lambdastar} \lambda_\star:=\inf\{\zeta<\lambda_0: G(\zeta)=0\}
\end{equation} as the left-most endpoint of the inflationary range containing $\lambda_0$. The theorem is as follows.  \begin{theorem}\label{thrm: LLN large} 
Under the same hypotheses as Theorem \ref{thm:lln moderate}, let $c$ be given by Theorem \ref{thm:lln moderate} and $M:=c\sqrt{n/\eta \log (n/\eta)}$. Define $\lambda_{\rm coll}<\lambda_0$ by \begin{equation}
	\label{eq: def lambda coll} \lambda_{\rm coll}:=\sup\left\{\zeta - \frac{(4-\eta \zeta s_i^2)(\eta \zeta s_i^2 -1)}{\eta s_i^2}(1-e^{-G(\zeta)/2}): \zeta \in [\lambda_\star, \lambda_0]; 1\le i\le m, \eta \zeta s_i^2>1\right\}
\end{equation}  For every $\lambda\in (\lambda_\star, \lambda_0)$, write $\lambda^+:=\max(\lambda, \lambda_{\rm coll})$. Then, there exists $C=C(c,\{(s_i, p_i)\})$ and $\alpha=\alpha(\lambda)>0$ so that  \begin{equation}
	\label{eq: nonperturbative LLN 1} \PP_{\mu_0, \lambda_0}(\TK{\lambda^+}>\TPM+k|\TPM<\infty ) \le  C (\lambda_0-\lambda)\log(n/\eta) e^{-\alpha k}\end{equation} Moreover, for fixed $M_-$, set \begin{equation} \label{eq: nonperturbative LLN 2}
	\tau^{\rm Pred}_{\downarrow, M_-}:=\inf\{t\ge \TPM: |\mu(t)| \le M_-\};
\end{equation} and \begin{equation} \label{eq: quantities in downwards hitting} q_-:=\min(G(\zeta): \zeta \in [\lambda^+, \lambda_0]);\qquad \delta:=\eta \min(|\zeta-(\eta s_i^2)^{-1}|: 1\le i\le m, \zeta\in [\lambda^+,\lambda_0])  \end{equation} which are both strictly positive. Then there exists $\vartheta=\vartheta(q_-,\delta)>0$ such that \begin{equation} \label{eq: dont come down} 
	\mathbb{P}_{\mu_0, \lambda_0}\left(\tau^{\rm Pred}_{\downarrow, M_-} < \TK{\lambda^+}\right) \le (M_-/M)^{\vartheta}.
\end{equation}
\end{theorem} \begin{remark}\label{rmk: why LLN large formulated} The display \eqref{eq: nonperturbative LLN 1} justifies the claim in Theorem \ref{thrm: informal} that the kernel reduction occurs in a time $\log((\lambda-\lambda_0)\log(n/\eta))+\mathcal{O}_{\mathbb{P}}(1)$. In the case either where $\lambda\approx \lambda_\star$ is close to the endpoint, or where $\delta\ll 1$ and the interval is close to a singularity, the exponent $\vartheta(q_-, \delta)$ may be very small due to some indexes with $|1-\eta \lambda s_i^2|\ll 1$, sampled with a low probability.  \end{remark} 

Before moving to the proof, we first explain the formulation, which corresponds to the key difficulty in the proof. We consider the behaviour of the process after $\TPM$, for $M=c\sqrt{n/(\eta \log(n/\eta))}$. Once $|\mu(t)|$ reaches the scale $\sim \sqrt{n/\eta}$, the dynamics \eqref{eq: SGD 2} show that the kernel $\lambda(t)$ falls by order $\mathcal{O}(1)$ at each time step. On the other hand, if the distance to $\lambda_\star<\lambda_0$ given by \eqref{eq: def lambdastar} is of the same scale $\mathcal{O}(1/\eta)$ as $\lambda$ itself, the prediction $\mu$ will continue to increase in absolute value with the kernel still above $\lambda_\star$. Once $|\mu|$ reaches the slightly larger scale $\mathcal{O}(\sqrt{n}/\eta)$, two different effects become important. The first is that every step decreases $\lambda(t)$ by $\mathcal{O}(1/\eta)$, comparable to $\lambda_0$. Secondly, the neglected quadratic term in the prediction updates in \eqref{eq: SGD 2} also becomes important. For indexes where $1-\lambda \eta s_i^2<0$, the quadratic increase, which we previously neglected in \eqref{eq: find the RW approximately} can lead to \begin{equation} \label{eq: collapse motivation} \left|1-\eta \lambda s_i^2 +\frac{\eta^2 s_i^4\mu(t)^2}{n}\right| \ll |1-\eta \lambda s_i^2|.\end{equation} Indeed, if this factor is $\ll 1$, it is possible for a single sample to terminate the spike, without necessarily reducing $\lambda(t)$ beyond $\lambda_\star$. We term this phenomenon {\em spike collapse}; an explicit configuration is given below. In the case where the spike is ended but remains in the inflationary regime, we can apply Theorem \ref{thm:lln moderate} to see that the process begins an additional spike and reduces the kernel further.    \paragraph{Example: Spike Collapse within the Inflationary Regime} \label{ex: spike collapse} Consider a one-point data set $s=100, p=1$ with hyperparameters $n=10^8, \eta=10^{-2}$, initialised at $\mu_0=0.682, \lambda_0=0.039$. In this case, there is no stochasticity, and the dynamics are simply the gradient descent. The inflationary regime is simply the catapult regime in \cite{lewkowycz2020large,zhu2022quadratic}, producing large spikes before stability when $\lambda\in (\lambda^{\rm MB}_{\rm crit}, \lambda^{\rm MB}_{\rm max})=(\frac1{50}, \frac1{25})$. Iterating the (now deterministic) relations \eqref{eq: SGD 2} produces a large spike at $\mu(6)=159.39, \lambda(6)=3.66 \times 10^{-2}$ before falling to $\mu(7)=-20.31$. However, this spike only reduces the kernel to $\lambda(7)=2.82 \times 10^{-2} > \lambda^{\rm MB}_{\rm crit}$, remaining in the inflationary regime. Another spike begins, reaching $\mu(10)=75.10$, before the curvature falls to $\lambda(10)=2.1\times 10^{-2}<\lambda^{\rm MB}_{\rm crit}$.  \\  We now give the
\begin{proof}[Proof of Theorem \ref{thrm: LLN large}] We first explain the strategy. The key observation is that, at any one time step, not all indexes can cause a spike collapse. We therefore define a time-dependent set of indexes $\cA(t)\subset\{1,2,\dots,m\}$ where an inequality like \eqref{eq: collapse motivation} holds. We introduce a stopping time $\tau_{\rm coll}$, which is the first time that such an index is sampled, and see that sampling such an index immediately drives the kernel below $\lambda_{\rm coll}\le \lambda^+$. These observations motivate us to define an auxiliary process $\widetilde{\mu}$ by modifying the process only when $i(t)\in \mathcal{A}(t)$, constructed in such a way that the log-drift remains positive, and that $\widetilde{\mu}$ provides a control on $\mu$ from below until $\tau_{\rm coll}$ occur, the kernel falls below $\lambda^+$, or an upper threshold $M^+$ is reached. The positive log-drift forces $\widetilde{\mu}$ to hit the upper threshold $M^+$ with high probability in a short time window, and this is sufficient to prove the probable decrease \eqref{eq: nonperturbative LLN 1}. Finally, optional stopping arguments show that $\widetilde{\mu}$ is unlikely to ever become small, which allows us to prove the claim \eqref{eq: dont come down}.  \step{1. Excluded Indexes and Collapse Time $\tau_{\rm coll}$} For $\mu\in \RR, \zeta\in [\lambda_\star, \lambda_0]$, we define a set $\cA(\mu, \zeta)$ of those datapoints close to collapse by \begin{equation} \label{eq: def excluded indexes} \cA(\mu, \zeta):=\left\{i: 1-\zeta \eta s_i^2<0, \frac{\mu^2 \eta^2 s_i^4}{n} \ge (1-e^{-G(\zeta)/2})(\eta \zeta s_i^2-1)\right\} \end{equation} and write, for shorthand, $\cA(t):=\cA(\mu(t), \lambda(t))$. Define further a stopping time $\tau_{\rm coll}$ for the first time a sample close to collapse is drawn: \begin{equation} \label{eq: def tcoll} \tau_{\rm coll}:=\inf\left\{t\in [\TPM, \TK{\lambda_\star}): i(t)\in \cA(t-1)\right\}. \end{equation} On the event $\{\tau_{\rm coll}<\infty\}$, we have $\lambda(\tau_{\rm coll}-1)\in [\lambda_\star, \lambda_0]$, and from \eqref{eq: SGD 2} and the definition of $\cA(t)$, the next update produces\begin{equation} \begin{split}
	\lambda(\tau_{\rm coll})& =\lambda(\tau_{\rm coll}-1)-\frac{\mu(\tau_{\rm coll}-1)^2 \eta}{n} s_{i(\tau_{\rm coll})}^2(4-\eta \lambda(\tau_{\rm coll}-1)s_{i(\tau_{\rm coll})}^2) \\[1ex] & \hspace{-1cm} \le \lambda(\tau_{\rm coll}-1)-\frac{(\eta \lambda(\tau_{\rm coll}-1)s_{i(\tau_{\rm coll})}^2-1)(4-\eta \lambda(\tau_{\rm coll}-1)s_{i(\tau_{\rm coll})}^2)}{\eta s_{i(\tau_{\rm coll})}^2}(1-e^{-G(\lambda(\tau_{\rm coll}-1)}) \\ & \le \lambda_{\rm coll} \end{split}
\end{equation} where the second line follows from the definition of $\cA(\mu, \zeta)$, and the final line follows because $\lambda(\tau_{\rm coll}-1)\in [\lambda_\star, \lambda_0]$ and using the definition of $\lambda_{\rm coll}$. Consequently, almost surely, \begin{equation}
	\label{eq: win by collapse} \TK{\lambda^+}\le \TK{\lambda_{\rm coll}} \le \tau_{\rm coll}.
\end{equation} \step{2. Coupling with Excluded Indexes.} Rather than the direct construction of supermartingales in the same way as in Lemma \ref{lemma: new hitting negative}, we first introduce a coupled process $\widetilde{\mu}(t)$, which is stochastically dominated by $\mu$ up to a stopping time $T$, and from which we can construct supermartingales. First, set an upper threshold \begin{equation} \label{eq: win by upper threshold}
	M^+:=\sqrt{\frac{n(\lambda_0-\lambda_\star)}{\eta a_-}}
\end{equation} for $a_-$ as defined in \eqref{eq: decrease in curvature}, and define the stopping time \begin{equation}
	\label{eq: end of big coupling} T:=\min\left(\tau_{\rm coll}, \TK{\lambda^+}, \tau^{\rm Pred}_{M^+}+1\right).
\end{equation}  For $q_-$ as defined in \eqref{eq: quantities in downwards hitting} we set, for each $t\ge \TPM$ and each index $i$, set \begin{equation} \label{eq: def bitilde}
	\overline{b}_i(t):=\begin{cases}
		\min\left(\left|1-\eta \lambda(t) s_i^2\right|,\left|1-\eta \lambda(t) s_i^2 + \frac{\mu(t)^2\eta^2 s_i^4}{n}\right|\right), & i\not \in \cA(t); \\ \left|1-\eta \lambda(t) s_i^2\right|, & i\in \cA(t).
	\end{cases}
\end{equation} and \begin{equation} \label{eq: def bibar} \widetilde{b}_i(t):=e^{-q_-/4}\overline{b}_i(t\land(T-1)).\end{equation} Recalling the Definition \eqref{eq: def excluded indexes} of $\cA(t)$, for every index $i$ it holds that \begin{equation}
	e^{-G(\lambda(t))/2} |1-\eta \lambda(t)s_i^2|\le \widetilde{b}_i(t)\le   |1-\eta \lambda(t)s_i^2|
\end{equation} which implies a uniform log-drift condition \begin{equation} \label{eq: lower bound moment nu} \sum_{i=1}^m p_i \log \widetilde{b}_i \ge \frac12 G(\lambda(t\land (T-1)))-\frac{q_-}4 \ge \frac12 \inf_{\zeta \in [\lambda, \lambda_0]} G(\zeta)-\frac{q_-}{4}=:\frac{q_-}4>0. \end{equation} Further, by the definition of $\delta$ in \eqref{eq: quantities in downwards hitting} and considering separately the cases $s_i^2\ge (1+\delta \eta \lambda_0)^{-1}$, $s_i^2\le (1+\delta \eta \lambda_0)^{-1}$, we find a bound, uniformly in $i, \zeta \in [\lambda^+, \lambda_0]$,\begin{equation}\label{eq: base lower bound 1} |1-\eta \zeta s_i^2| \ge \frac{\delta}{1+\delta \eta \lambda_0}=: c_1(\delta) \end{equation} which, together with the standing hypothesis Assumption \ref{hyp: assumption linear model}, implies almost sure uniform upper and lower bounds \begin{equation}\label{eq: upper lower btilde}
	c_1(\delta) \le \widetilde{b}_i(t) \le 3; \qquad 1\le i\le m, \qquad  t\ge \TPM. 
\end{equation} We now define a process $\widetilde{\mu}(t), t\ge \TPM$, by setting \begin{equation}
	\label{eq: def Xbar} \widetilde{\mu}(t):=|\mu(\TPM)|\prod_{\TPM< s\le t} \widetilde{b}_{i(s)}(s).
\end{equation}The Definition \eqref{eq: def bitilde} and the update rule \eqref{eq: SGD 2} together imply the bound \begin{equation} \label{eq: coupling 2}
|\mu(t)|\ge \widetilde{\mu}(t)e^{q_-(t-\TPM)/4}
\end{equation} valid for all $\TPM\le t< T$ with probability 1. \step{3. Supermartingales constructed from $\widetilde{\mu}$} We now show that there exists $\vartheta$, with {\em only} the dependencies advertised in the theorem, such that $\widetilde{\mu}^{-\vartheta}$ is a supermartingale, in order to later use optional stopping arguments similar to \eqref{eq: opt stop UB}. Following the same logic as (\ref{eq: submartingale case 1} -  \ref{eq: submartingale case 2}), it is sufficient to find $\vartheta>0$ such that the (random, time-dependent) function \begin{equation} \label{eq: def F}
	F(t,\theta):= \EE_{\mu_0, \lambda_0}[\widetilde{b}_{i(t+1)}(t)^{-\theta}|\cF_t, \TPM\le t]
\end{equation} satisfies $F(t,\vartheta)\le 1$ almost surely for all $t\ge \TPM$. The claim then follows from the Definition \eqref{eq: def Xbar}. We note that $\tilde{b}_i(t)$ are all $\cF_t$-measurable because $T$ is a stopping time, which also ensures the adaptedness of $\widetilde{\mu}(t)^{-\theta}$ for any $\theta\in \RR$. In the remainder of this step, all statements are almost sure, simultaneously in all $t\ge \TPM$.\\ \\ By definition, $F(t,0)=1$, and due to the upper and lower bounds \eqref{eq: upper lower btilde} imply that $F(t,\cdot)$ is smooth in $\theta\in \RR$. Differentiating under the integral sign, the first derivative at zero is \begin{equation} \label{eq: first derivative bound}
\left.	\frac{d}{d\theta}\right|_{\theta=0} F(t,\theta)=-\sum_{i=1}^m p_i \log \widetilde{b}_i(t) \le -\frac{q_-}{4}
\end{equation} thanks to \eqref{eq: lower bound moment nu}, while the second derivative is bounded above on $[0,1]$ by \begin{equation}
	\label{eq: second derivative bound} \frac{d^2}{d\theta^2} F(t,\theta)=\sum_{i=1}^m p_i(\log \widetilde{b}_i)^2 \widetilde{b}_i^{-\theta}\le c_1(\delta)^{-1} \max(|\log c_1(\delta)|, 3)^2
\end{equation} for $c_1(\delta)$ given by \eqref{eq: base lower bound 1}. Consequently, we obtain $F(t,\vartheta)\le 1$ for $$ \vartheta(q_-,\delta)=\min\left(1, \frac{q_-c_1(\delta)}{4\max(|\log c_1(\delta)|,3)^2 }\right).$$ \step{4. Hitting estimates via Supermartingales} We now use the supermartingales constructed in the previous step to prove, for all $0<h\le M$, \begin{equation}
	\label{eq: hitting downwards} \PP_{\mu_0, \lambda_0}\left(\inf_{t\ge \TPM} \widetilde{\mu}(t)\le h\hspace{0.1cm}\bigg|\hspace{0.1cm}\TPM<\infty\right) \le \left(\frac{h}{M}\right)^\vartheta
\end{equation} where the exponent $\vartheta$ is given by Step 3. This follows at once by applying optional stopping to the non-negative supermartingale $\widetilde{\mu}(t)^{-\vartheta}$ and the stopping time $T_h:=\inf\{t\ge \TPM: \widetilde{\mu}(t)\le h\}$, and arguing as in \eqref{eq: opt stop UB}.  \step{5. High-Probability Estimate for $\TK{\lambda^+}$.} We now deduce the estimate \eqref{eq: nonperturbative LLN 1} for the time to achieve the kernel decrease from the coupling in Step 2 and the hitting probability in Step 4. From the constructions \eqref{eq: win by collapse} and \eqref{eq: win by upper threshold}, it follows that \begin{equation}
	\label{eq: coupling only ends one way} \TK{\lambda^+}\le 1+\min(\tau^{\rm Pred}_{M^+}, \tau_{\rm coll}).
\end{equation}As a result, by considering separately the cases where the coupling \eqref{eq: coupling 2} remains intact and has failed, it follows that \begin{equation} \label{eq: key coupling fact} \TK{\lambda^+} \le 1+ \inf\left\{t\ge \TPM: |\widetilde{\mu}(t)|e^{q_-(t-\TPM)/4}\ge M^+\right\}. \end{equation} Consequently, for any $k\ge 0$, \eqref{eq: hitting downwards} shows that, for the exponent $\vartheta=\vartheta(q_-, \delta)>0$ produced by Step 3, \begin{equation}\begin{split} \PP\left(\TK{\lambda^+}> \TPM +k +1  |\TPM<\infty\right) &\le \mathbb{P}\left(\widetilde{\mu}(\TPM+k) \le M_+ e^{-q_- k/4}|\TPM<\infty \right) \\& \le e^{-q_- k\vartheta/4}\left(\frac{M_+}{M}\right)^{\vartheta} \\& = e^{-q_- k \vartheta/4} \left(\frac{(\lambda_0-\lambda_\star)\log (n/\eta)}{ c a_-}\right)^\vartheta. \end{split} \end{equation} The claim \eqref{eq: nonperturbative LLN 1} follows by absorbing the data-dependent constants, using that $\vartheta\le 1$ for the prefactors, and shifting $q$. \medskip   \step{6. No Descent to $M_-$ without kernel decrease.} We finally prove \eqref{eq: dont come down}. As a result of \eqref{eq: end of big coupling} and \eqref{eq: coupling only ends one way}, it follows that, for all $t<\TK{\lambda^+}$, the coupling \eqref{eq: coupling 2} remains intact, which implies the containment of events \begin{equation}
	\label{eq: dont come down 2} \left\{\tau^{\rm Pred}_{\downarrow, M_-}<\TK{\lambda^+}\right\} \subset \left\{\inf_{t\ge \TPM} \widetilde{\mu}(t)\le M_-\right\}.
\end{equation} The claimed bound now follows immediately from \eqref{eq: hitting downwards}, up to changing $\vartheta$ to $\vartheta/4$.   \end{proof}

We finally turn to the deflationary case. \begin{theorem} \label{thrm: LDP large spikes} Let $\lambda^{\rm MB}_{\rm crit}<\lambda_0<\lambda^{\rm MB}_{\rm max}$ and, for $\delta>0$, let \begin{equation}
	\label{eq: Idelta} I_\delta:=\{i: |1-\eta \lambda_0 s_i^2|>1+\delta\};\qquad p_\delta:=\sum_{i\in I_\delta}p_i.
\end{equation} Then, for any $\delta>0$, for all $\lambda$ in the range \begin{equation}\label{eq: range of lambda}
	\lambda_0-\delta \lambda^{\rm MB}_{\rm max} \min\left(\frac14, \frac{a_-}{a_+}\right)\le \lambda\le \lambda_0
\end{equation} and any $M\ge |\mu_0|$, $M_-<M$, it holds that \begin{equation}\label{eq: final LDP LB large spikes}
	\PP_{\mu_0,\lambda_0}\left(\TK{\lambda}\le \tau^{\rm Pred}_{\downarrow, M_-}|\TPM<\infty\right)\ge p_{3\delta}\left(\frac{n(\lambda_0-\lambda)}{\eta a_- M^2}\right)^{-\alpha(\delta)} \end{equation} where $\tau^{\rm Pred}_{\downarrow, M_-}$ is as in \eqref{eq: nonperturbative LLN 2}, and the exponent is given explicitly by \begin{equation} \label{eq: thetadelta} \alpha(\delta):=-\frac{\log p_{3\delta}}{2\log(1+\delta)}.
\end{equation} \begin{remark}
	Together with Theorem \ref{thrm: LDP large spikes} and the strong Markov property, this justifies the claim \eqref{eq: large spike informal}. We have not attempted to optimise the proof strategy or the exponent achieved, preferring to give a final bound only in terms of quantities which can be computed and interpreted directly from $\lambda_0, \eta, \{(s_i, p_i)\}$. Moreover, in combining \eqref{eq: final LDP LB large spikes} and Theorem \ref{thrm: LDP large spikes}, we may choose the largest scale \eqref{eq: restrict to scales} allowed by Theorem \ref{thm: LDP spikes}: $$M\sim \frac{1}{\log^{1/\beta}(n/\eta)}\sqrt{\frac{n}{\eta}}.$$  With this choice, the sup-optimal exponent $\alpha$ only enters the $n, \eta$ dependence through a logarithm $\log^\alpha(n/\eta)$. \end{remark}	
\end{theorem} \begin{proof} Fix $\delta>0$ and let $p_\delta, I_\delta, \lambda$ be as given. Consider, for $k$ to be chosen later, the event \begin{equation}
	A_{k, 3\delta}:=\left\{\TPM<\infty; i(\TPM+1),\dots, i(\TPM+k)\in I_{3\delta}\right\}
\end{equation} which has probability, conditional on $\{\TPM<\infty\}$, equal to $p_{3\delta}^k$. We now set an upper threshold as in \eqref{eq: win by upper threshold} \begin{equation}
	M^+:=\sqrt{\frac{n(\lambda_0-\lambda)}{a_-\eta}}.
\end{equation}  On the event $A_{k, 3\delta}$, for $\TPM+1\le t< \min(\TPM+k+1, \TK{\lambda},\tau^{\rm Pred}_{M^+})$, we estimate \begin{equation} \eta(\lambda_0-\lambda)s_{i(t)}^2 \le \frac{\delta \lambda^{\rm MB}_{\rm max}}{4} (\eta s_\star^2)=\delta \end{equation} by the definition of $\lambda^{\rm MB}_{\max}$ and the first part of \eqref{eq: range of lambda}, and similarly \begin{equation}
	\frac{\eta^2 s_{i(t)}^4 \mu(t)^2}{n} \le \frac{(M^+)^2 \eta^2 s_\star^4}{n} = \frac{\eta s_\star^4(\lambda_0-\lambda)}{a_-} \le \delta
\end{equation} using the second part of the lower bound in \eqref{eq: range of lambda}, and the definitions of $\lambda^{\rm MB}_{\rm max}, a_\pm$. Gathering the previous two displays and using the definition of $I_{3\delta}$, for the same range of $t$ we have \begin{equation}
	\label{eq: bound multiplicative factor below} \begin{split} & \left|1-\eta \lambda(t)s_{i(t)}^2 + \frac{\eta^2 s_{i(t)}^4 \mu(t)^2}{n}\right| \\ & \hspace{2cm} \ge \left|1-\eta \lambda_0 s_{i(t)}^2\right| - |\eta (\lambda_0 - \lambda) s_{i(t)}^2| - \frac{\eta^2 s_{i(t)}^4 \mu(t)^2}{n}  \\ &  \hspace{2cm} \ge (1+3\delta)-\delta - \delta = 1+\delta. \end{split}
\end{equation} Choosing now $k:=\lceil \log(M^+/M) / \log(1+\delta)\rceil$, we see that on the event $A_{k, 3\delta}$, either $\tau^{\rm Pred}_{M_+}$ or $\TK{\lambda}$ must occur no later than $\TPM+k$, and that $|\mu(t)|$ is increasing until the first of these occurs. Arguing as in \eqref{eq: win by upper threshold}, it holds almost surely that $\TK{\lambda}<\tau^{\rm Pred}_{M^+}$ by definition of $M^+$, and so in either case we achieve $\TK{\lambda}\le \tau^{\rm Pred}_{\downarrow, M_-}$ for any $M_-<M$. With this choice of $k$, we use $$ k\le 1+\frac{\log(M^+/M)}{\log(1+\delta)}$$ to find a bound \begin{equation}\begin{split}
	\label{eq: tada} \PP_{\mu_0, \lambda_0}(A_{k, 3\delta}|\TPM<\infty) & \ge p_{3\delta} \exp\left(\log p_{3\delta} \frac{\log(M^+/M)}{\log(1+\delta)}\right)  = p_{3\delta} \left(\frac{M_+}{M}\right)^{-2\alpha(\delta)}\end{split} 
\end{equation} for $\alpha(\delta)$ given by \eqref{eq: thetadelta}. Substituting in the definition of $M^+$ produces the claim \eqref{eq: final LDP LB large spikes}, and the proof is complete. 
\end{proof}

\section{Extension to ReLU Activation} \label{sec: RELU} Finally, we show how all of the arguments of the rest of the paper may be modified if the activation function $\varphi$ of the first layer in \eqref{eq: toy model} is taken as the rectified linear unit $\varphi(w)=\sigma(w)=\max(0,w)$, giving the network \begin{equation}
		F(\Theta, s):=\frac{1}{\sqrt{n}} \sum_{r=1}^n a_r\sigma(w_r s).
	\end{equation} We first note that there are now two sets of relevant predictions and kernels, given by \begin{equation}\label{eq: relu pred and kernel}
	\mu^{\pm}(\Theta)=F(\Theta, \pm 1); \qquad \lambda^{\pm}(\Theta)=\frac1n\sum_{r=1}^n (\sigma( \pm w_r)^2 + \indiq(\pm w_r>0) a_r^2). 
\end{equation} The corresponding NTK is \begin{equation}\label{eq: relu NTK} K(\Theta;x,y)=\lambda^+(\Theta)\indiq(x>0, y>0)+\lambda^-(\Theta)\indiq(x<0, y<0). \end{equation} Here, and in the sequel, we follow the convention \cite{du2018gradient} that the derivative of the ReLU function $\sigma'(w)$ refers to any element of the subdifferential, which allows us to make an arbitrary choice at $w=0$. Repeating the calculations leading to \eqref{eq: SGD 2}, we now find that, on time steps where $s_{i(t+1)}>0$, \begin{equation} \label{eq: SGD Relu update}
		\begin{cases}
			w_r(t+1)=w_r(t)-\frac{\eta}{\sqrt{n}}\mu^+(t)a_r(t)\indiq(w_r(t)>0)s_{i(t+1)}^2; \\[2ex] a_r({t+1})=a_r(t)-\frac{\eta}{\sqrt{n}}\mu^+(t)\sigma(w_r(t))s_{i(t+1)}^2.
		\end{cases}
	\end{equation} Correspondingly, still for steps where $s_{i(t+1)}>0$, the prediction and kernels are updated to  \begin{equation} \label{eq: update ReLU model prediction}
		\begin{split} & \mu^+({t+1})=\left(1-2\eta \lambda^+(t) s_{i(t+1)}^2 + \delta^{\rm Pred}(t+1) + \frac{\eta^2 (\mu^+(t))^2 s_{i(t+1)}^4}{n}  \right) \mu^+(t); \\ & \lambda^+({t+1}) = \lambda^+(t) +\frac{(\mu^+(t))^2\eta}{n}(\eta \lambda(t)s_{i(t+1)}^4 - 4s_{i(t+1)}^2) -\delta^{\rm Ker}(t+1). \end{split}
	\end{equation} Relative to \eqref{eq: SGD 2}, error terms $\delta^{\{{\rm Pred, Ker}\}}$ appear in the updates to the prediction and kernel respectively, which result from the contribution of neurons $w_r$ where $w_r(t+1)<0<w_r(t)$:\begin{equation}\begin{split}
			\label{eq: error update prediction} \hspace{-1cm}\delta^{\rm Pred}(t):=\frac{\mu^+(t)^{-1}}{\sqrt{n}} &\sum_{r=1}^n \left(\sigma(w_r(t+1))-\sigma(w_r(t))-\indiq(w_r(t)>0)(w_r(t+1)-w_r(t)\right) a_r(t+1); \end{split}
		\end{equation} \begin{equation}
			\label{eq: kernel update prediction} \delta^{\rm Ker}(t)=\frac1n\sum_{r=1}^n[(a_r(t+1))^2+(w_r(t+1))^2]\indiq(w_r(t)>0\ge w_r(t+1)) \ge 0.
		\end{equation}At the same time, the change in the activation pattern produces a change in $\mu^-, \lambda^-$ by \begin{equation}\begin{split}
			\label{eq: SGD relu update other} & \mu^-(t+1)=\mu^-(t)-\delta^{\rm Pred}(t+1)\mu^+(t); \\ & \lambda^-(t+1)=\lambda^-(t)+\delta^{\rm Ker}(t+1).
		\end{split}\end{equation} Identical formul{\ae} apply if $s_{i(t+1)}<0$, reversing every instance of the labelling $\pm$. \\ \\ The key insight, formulated in Proposition \ref{prop: change activation} below, is that, if a certain asymmetric initialisation is imposed on $w_r(0), a_r(0)$, then no activations $w_r$ flip sign until at least one of $\mu^{\pm}$ has become large. In particular, the errors $\delta^{\rm Pred, Ker}(t)$ are identically zero until a spike has occurred in at least one of $\mu^\pm$. Up to this time, $\mu^+(t), \lambda^+(t)$ follow the dynamics \eqref{eq: SGD 2}, with $s_i$ replaced by $$ s^+_i:=\max(0, s_i).$$  The same also holds for $(\mu^-, \lambda^-)$ with $s^-_i:=\max(-s_i, 0)$, and the two parts are further independent, aside from the condition that only one of them is updated at each step.  The assertions of Theorem \ref{thrm: informal2}, which we state rigorously in Theorem \ref{thrm: RELU} below, now follow immediately from Theorems \ref{thm:lln moderate}, \ref{thm: LDP spikes}. \\ \\ In order to make this rigorous, we define \begin{definition}[$w$-biased initialisation] \label{def: w biased} We say that a probability measure $\PP$ is $w$-biased if, with probability 1, $|w_r(0)|\ge |a_r(0)|$ for every $1\le r\le n$.\end{definition} We refer also to \cite[Assumption 1]{dana2025convergence}, \cite{boursier2024simplicity,boursier2025early} for similar assumptions. As discussed above, the result claimed is only valid until (at least) one of $|\mu^{\pm}|$ becomes large. We therefore define $\tau^{\rm Pred, \pm}_M$ in the same way as \eqref{eq: define taupred linear model} for $\mu^{\pm}$ in place of $\mu$. With this fixed, the proposition is as follows. \begin{proposition}\label{prop: change activation} Let $\PP$ be any $w$-biased probability measure under which are defined the parameters $\Theta(t)=(w_r(t), a_r(t))$, updated according to \eqref{eq: SGD Relu update}. Then there exists $c>0$ such that \begin{equation}
			\label{eq: change activation conclusion} \PP\left(\sgn(w_r(t))=\sgn(w_r(0))\text{ for all }1\le r\le n, 0\le t\le \tau^{\rm Pred, +}_{c\sqrt{n}/\eta}\land \tau^{\rm Pred, -}_{c\sqrt{n}/\eta} \right)=1.
		\end{equation} \end{proposition} The key idea, see \eqref{eq: balance layers} below, is not new, and is essentially due to \cite{wojtowytsch2020convergence}. From this and the preceding discussion, we obtain the following result. \begin{theorem}\label{thrm: RELU} In the notation of Assumption \ref{hyp: assumption linear model} (2), suppose further that $\PP_{\mu_0^{\pm}, \lambda^{\pm}_0}$ is $w$-biased, and let $s^{\pm}_i, \lambda^{\rm MB, \pm}_{\rm crit}, \lambda^{\rm MB, \pm}_{\rm max}, G^{\pm}, \vartheta^\pm$ be as in Theorem \ref{thrm: informal}. Assume, throughout, that $\lambda_0^\pm < \lambda^{\rm MB, \pm}_{\rm max}$. Then \begin{enumerate}[label=\alph*).]
			\item {\em [At least one component inflationary:]} If  $G^+(\lambda_0^+)>0$, then there exists $\beta>0$ and, for every $0<\delta<1$, there exists $\kappa$ such that, for all $|\mu_0^+|\le M\le \kappa\sqrt{n/\eta}$ which additionally satisfy \begin{equation} \label{eq: additional restriction on M}
				\frac{G^-(\lambda_0^-)(1+\delta)\log(M/|\mu_0^+|)}{G^+(\lambda_0^+)} \le \log\left(\frac{\kappa}{|\mu_0^-|\log^{1/\beta}(n/\eta)}\sqrt{\frac{n}{\eta}}\right)	\end{equation} and $\vartheta>0$ such that \begin{equation}
					\label{eq: RELU inflationary} \begin{split} &\PP_{\mu_0^{\pm},\lambda_0^{\pm}}\left(\tau^{\rm Pred, +}_M\not \in \left[\frac{(1-\delta)\log(M/|\mu_0^+|)}{G^+(\lambda_0^+)},\frac{(1+\delta)\log(M/|\mu_0^+|)}{G^+(\lambda_0^+)}\right]\right) \\ & \hspace{3cm}\le \left(\frac{|\mu_0^+|}{M}\right)^\vartheta+\left(\frac{\kappa}{|\mu_0^-|\log^{1/\beta}(n/\eta)}\sqrt{\frac{n}{\eta}}\right)^{-\vartheta}+\exp\left(-\beta\frac{n}{\eta M^2}\right). \end{split}
				\end{equation} The same holds for $\mu^-(t)$ if $G^-(\lambda^-_0)>0$, with every $\pm$ reversed. \item {\em [Both components deflationary:]} Suppose instead that $G^\mathfrak{s}(\lambda^{\mathfrak{s}}_0)<0$ for both $\mathfrak{s}\in \{\pm\}$, and that at least one of $\lambda^\mathfrak{s}_0>\lambda^{\rm MB, \mathfrak{s}}_{\rm crit}$. Then, for any $\vartheta^\mathfrak{s}_-<\vartheta^\mathfrak{s}(\lambda^\mathfrak{s}_0)<\vartheta^\mathfrak{s}_+$, there exist $c, \beta>0$, $C<\infty$ such that, for any $$ |\mu^\mathfrak{s}_0|\le M \le c \sqrt{\frac{n}{\eta \log^\beta(n/\eta)}}$$ it holds that \begin{equation} \label{eq: RELU deflationary}C^{-1}\left(\frac{|\mu_0^-|}{M}\right)^{\vartheta^+_+}\le \PP_{\mu_0^{\pm}, \lambda_0^\pm}\left(\tau^{\rm Pred,+}_M<\infty\right) \le C \left(\frac{|\mu_0^-|}{M}\right)^{\vartheta^+_-}+C\left(\frac{c}{|\mu_0^-|\log^{1/\beta}(n/\eta)}\sqrt{\frac{n}{\eta}}\right)^{\vartheta^-_-}\end{equation} and the same holds when every instance of the signs $\pm$ are reversed.
		\end{enumerate}
			
		\end{theorem} \begin{proof}[Sketch Proof of Theorem \ref{thrm: RELU}] In item a), where $\mu^+$ is assumed to be in the inflationary regime, we can distinguish the cases depending on whether $G^-(\lambda^-_0)$ is non-negative or negative: The additional restriction \eqref{eq: additional restriction on M} is vacuous in the latter case. Let $\kappa=c$ be given by Proposition \ref{prop: change activation}. In either case, there exists $\widetilde{\mu}^+, \widetilde{\lambda}^+$, governed (for all $t\ge 0)$ by the dynamics \eqref{eq: SGD 2}, which agrees with $\mu^+, \lambda^+$ up to the time \begin{equation}\label{eq: stopping relu}T:=\tau^{\rm Pred, +}_{c\sqrt{n}/\eta}\land \tau^{\rm Pred, -}_{c\sqrt{n}/\eta}. \end{equation} Let $\widetilde{\TPM}$ be the hitting time \eqref{eq: define taupred linear model} for the coupled process $\widetilde{\mu}^+$. Writing $I_\delta$ for the interval appearing in \eqref{eq: RELU inflationary}, we have the containment \begin{equation}
			\left\{\tau^{\rm Pred, +}_M\not \in I_\delta \right\} \subset \{\widetilde{\TPM}\not \in I_\delta\} \cup \left\{T=\tau^{\rm Pred,-}_{c\sqrt{n}/\eta}<\frac{(1+\delta)\log(M/|\mu_0^+|)}{G^+(\lambda_0^+)}.\right\} 
		\end{equation} The probability of the first event is estimated by Theorem \ref{thm:lln moderate}, producing the first and third terms of \eqref{eq: RELU inflationary}, while on the second event, the coupling of $\mu^-$ to a copy of \eqref{eq: SGD 2} on the bases $s^-_i$ remains exact until $T$. The probability of the second event may therefore either be estimated by Theorem \ref{thm:lln moderate} if $G^-(\lambda_0^-)>0$, or Theorem \ref{thm: LDP spikes} if $G^-(\lambda_0^-)<0$, using in either case a threshold, possibly changing $\beta$, $$ M^-=\frac{\kappa}{\log^{1/\beta}(n/\eta)}\sqrt{\frac{n}{\eta}} \le \kappa \frac{\sqrt{n}}{\eta}$$ and using, if additionally $G^-(\lambda_0^-)>0$, the extra condition \eqref{eq: additional restriction on M} to ensure that the hitting time for the process $\widetilde{\mu}^-$ coupled to $\mu^-$ to reach $M^-$ is, with high probability, outside $I_\delta$. In either case, we produce the second term of \eqref{eq: RELU inflationary}, and the claim is proven. \\ \\ In item b), where both components are in the deflationary regime, we use the same coupling argument and Theorem \ref{thm: LDP spikes} to see that, for $\mathfrak{s}\in \{\pm\}$ and any $\vartheta^\mathfrak{s}_-<\vartheta^\mathfrak{s}$, there exists $c\le \kappa,C$ such that, for \begin{equation} \label{eq: deflationary range of scales RELU}M\le \frac{c}{\log^{1/\beta}(n/\eta)}\sqrt{\frac{n}{\eta}},\end{equation} it holds that \begin{equation} \label{eq: couple deflationary}
			\PP_{\mu_0^{\pm}, \lambda_0^{\pm}}\left(\tau^{{\rm Pred}, \mathfrak{s}}_M < T\right) \le C\left(\frac{|\mu^\mathfrak{s}_0|}{M}\right)^{\vartheta^{\mathfrak{s}}_-}.
		\end{equation} As in the previous part, let $\widetilde{\mu}^+$ be governed by \eqref{eq: SGD 2} for $\lambda_0^+, s_i^+$, so that it agrees with $\mu^+$ up to time $T$. For the upper bound of \eqref{eq: RELU deflationary}, we have the containment \begin{equation}\begin{split}
			\{\tau^{\rm Pred,+}_M<\infty\}&\subset \{\widetilde{\TPM}<\infty\} \cup\{\tau^{\rm Pred,-}_{\kappa\sqrt{n}/\eta}<T\} \\ &\subset \{\widetilde{\TPM}<\infty\} \cup\{\tau^{\rm Pred,-}_{M^-}<T\}. \end{split}
		\end{equation} Choosing $M^-$ to be the maximum value allowed by \eqref{eq: deflationary range of scales RELU}, the probabilities of both events may be bounded by \eqref{eq: couple deflationary}. For the lower bound, we can repeat the arguments leading to the lower bound of Theorem \ref{thm: LDP spikes} directly, using martingales depending only on $\mu^+$ and noting that, under the tilted measure, $\mu^-$ has the same (deflationary) dynamics and, with high probability, does not become large.
			
		\end{proof} We conclude with the \begin{proof}[Proof of Proposition \ref{prop: change activation}] From \eqref{eq: SGD Relu update}, on draws where $s_{i(t+1)}>0$, we compute \begin{equation}\begin{split} \label{eq: balance layers}
			|w_r(t+1)|^2-|a_r(t+1)|^2&=|w_r(t)|^2-|a_r(t)|^2 \\ & + \frac{\eta^2}{n}(\mu^+(t))^2s_{i(t+1)}^4(|a_r(t)|^2-|w_r(t)|^2)\indiq(w_r(t)s_{i(t+1)}>0) \\ & \hspace{-2cm}=\left(1-\frac{\eta^2}{n}(\mu^+(t))^2s_{i(t+1)}^4 \indiq(w_r(t)s_{i(t+1)}>0)\right)(|w_r(t)|^2-|a_r(t)|^2)
		\end{split}\end{equation} with an equivalent formula for when $s_{i(t+1)}<0$. By hypothesis, $|w_r(0)|^2-|a_r(0)|^2\ge 0$ for all neurons. The multiplicative factor in the update rule above can only become negative when $$ \frac{(\mu^{\pm}(t))^2s^4_\star\eta}{n}>1 $$ which is only possible when one of $\mu^\pm(t)$ exceeds $\sqrt{n}/(\eta s^2_\star)$, implying that $$t\ge \tau^{\rm Pred,+}_{\sqrt{n}/(\eta s^2_\star)}\land \tau^{\rm Pred, -}_{\sqrt{n}/(\eta s^2_\star)}.   $$  At all times up to and including this final time, repeated uses of \eqref{eq: balance layers}, and its counterpart for  draws with $s_{i(t+1)}<0$, write $|w_r(s)|^2-|a_r(s)|^2$ as a product of non-negative factors for all $1\le r\le n$, which implies that $|a_r(s)|\le |w_r(s)|$ up to this time. Returning to \eqref{eq: SGD Relu update}, the only way in which $w_r$ can change sign is when $$ \left|\frac{\eta}{\sqrt{n}}\mu^{\pm}(t) a_r(t) s_{i(t+1)}^2\right|\ge |w_r(t)|$$ which is only possible if either $|a_r(t)|\ge |w_r(t)|$ or $|\mu^{\pm}(t)|\ge \sqrt{n}/(\eta s_\star^2)$, and the proof is complete. \buildingsite
			
		\end{proof}

\appendix
\section{Proof of Theorem \ref{thm:lln moderate}}\label{appendix: mod inf} We now give the proof of Theorem \ref{thm:lln moderate}, which was deferred in the main text. The key arguments are similar to those used in the proof of Theorem \ref{thm: LDP spikes}, to which we will refer back as needed.
\begin{proof}For clarity, we divide into steps. \step{1. Parameter Choices.} First, given $\lambda_0$ and a window parameter $\delta>0$, we choose $\epsilon, \kappa>0$ such that the upper and lower drifts \begin{equation} \label{eq: upper lower drifts def}
	G^{\pm}(\lambda_0; \epsilon, \kappa):=\sum_{i=1}^m p_i \log b^{\pm}_i(\epsilon, \kappa, \lambda_0)
\end{equation} for the shifted logarithmic bases $b^{\pm}_i$ given by (\ref{eq: def bi+} - \ref{eq: def bi-}) satisfy \begin{equation}
	\label{eq: upper lower drifts} \frac{1}{\sqrt{1+\delta}}G(\lambda_0)<G^-(\lambda_0)<G(\lambda)<G^+(\lambda)<\frac1{\sqrt{1-\delta}}G(\lambda)
\end{equation} and such that $\max_i b^+_i\le 4$. We let $\alpha>0$ be given by \begin{equation} \label{eq: def gamma appendix proof} \alpha:=\sum_{i=1}^m p_i (b^-_i)^2-1 >0 \end{equation} which is strictly positive as a result of \eqref{eq: upper lower drifts} and Jensen's inequality. With this choice of $\epsilon, \kappa$, we choose \begin{equation} \label{eq: clever choice of c} c=\min\left(1, \kappa, \sqrt{\frac{\epsilon G(\lambda_0)}{2s_\star^2}}\right)\end{equation} Throughout, we write $T$ for the stopping time \begin{equation}
	\label{eq: introduce T} T:=\TPM \land \TK{\lambda_0-\epsilon}.
\end{equation} \step{2. Controlling Decrease of the Kernel} In this step, we show how the arguments of Proposition \ref{prop: slow escape} may be repeated and, together with the choice \eqref{eq: clever choice of c}, lead to an estimate \begin{equation} \label{eq: khasminskii conclusion 2}
	\PP_{\mu_0, \lambda_0}(\TK{\lambda_0-\epsilon}<\TPM)\le 2 \exp\left(-\gamma \frac{n}{\eta M^2}\right).
\end{equation} As in Lemma \ref{lemma: new hitting negative}, for $t<T$ the prediction satisfies \begin{equation}
	\label{eq: repeat khasminskii 1} \EE_{\mu_0, \lambda_0}[|\mu(t+1)|^2|\cF_t, t<T]>(1+\alpha) |\mu(t)|^2
\end{equation} where $\alpha>0$ is given by \eqref{eq: def gamma appendix proof}. It follows that \begin{equation}
	\label{eq: supermg repeat khasminskii} Z_t:=\delta \sum_{s< t} |\mu(s)|^2 - |\mu(t)|^2
\end{equation} makes $Z^T$ into a supermartingale. Since $Z_0\le 0$, it follows that $\EE_{\lambda_0, \mu_0} Z_T\le 0$ by optional stopping, producing \begin{equation}
	\label{eq: pre khasminski 2} \frac{\alpha}{32 M^2}\EE_{\mu_0, \lambda_0}\sum_{s<T} |\mu(s)|^2 \le \frac12.
\end{equation} Using the Khas'minskii lemma \cite{stummer2000exponentials} as in Proposition \ref{prop: slow escape}, and that $|\mu(T)|\le 4M$, we find \begin{equation}
	\mathbb{E}_{\mu_0, \lambda_0}\left[\exp\left(\frac{\alpha}{32 M^2}\sum_{s<T} |\mu(s)|^2\right)\right]\le 2
\end{equation} whence \begin{equation}
	\PP_{\lambda_0, \mu_0}\left(\TK{\lambda_0-\epsilon}<\TPM\right)\le 2\exp\left(-\frac{\alpha \epsilon n}{32 a_+ \eta M^2}\right)
\end{equation} which is of the form desired \eqref{eq: khasminskii conclusion 2}, with $$ \gamma:=\frac{\alpha \epsilon}{32 a_+ \eta}.$$\\ \step{3. Tail Estimates}  Finally, we show that $\vartheta>0$ may be chosen so that \begin{equation} \label{eq: upper tail}
	\PP_{\mu_0, \lambda_0}\left(\TPM<(1-\delta)\frac{\log(M/|\mu_0)}{G(\lambda_0)}, \TK{\lambda_0-\epsilon}\ge \TPM\right) \le (|\mu_0|/M)^{\vartheta};
\end{equation}\begin{equation} \label{eq: lower tail}
	\PP_{\mu_0, \lambda_0}\left(\TPM>(1+\delta)\frac{\log(M/|\mu_0)}{G(\lambda_0)}, \TK{\lambda_0-\epsilon}\ge \TPM\right) \le (|\mu_0|/M)^{\vartheta}
\end{equation} which, together with \eqref{eq: khasminskii conclusion 2}, produce the claim in the theorem. We prove \eqref{eq: upper tail}; the claim \eqref{eq: lower tail} follows by the same reasoning. Letting $t_-$ be the time $$ t_-:=(1-\delta)\frac{\log(M/|\mu_0|)}{G(\lambda_0)}$$ appearing in \eqref{eq: upper tail}, we start by observing that, on the event in \eqref{eq: upper tail}, we have for any $\zeta \in (1, (1-\delta)^{-1/2})$, \begin{equation}\label{eq: premg estimate} \sup_{t\le t_-}\left(\log\frac{|\mu(t)|}{|\mu_0|}-t\zeta G^+(\lambda_0)\right)\ge \log(M/|\mu_0|)\left(1-\zeta(1-\delta)\frac{G^+(\lambda_0)}{G(\lambda_0)}\right) \end{equation} and, writing $\alpha(\zeta)$ for the final parenthesis, $\alpha(\zeta)>0$ for $\zeta<(1-\delta)^{-1/2}$ thanks to \eqref{eq: upper lower drifts}. As soon as $\zeta>1$, there exists positive $\theta=\theta(\zeta)>0$ such that $$ \sum_{i=1}^m p_i(b^+_i)^\theta e^{-\theta \zeta G^+(\lambda_0)}=1$$ which implies, as in Lemma \ref{lemma: new hitting negative}, that $$ Z_t:=\left(\frac{|\mu(t\land T)|}{|\mu_0|}\right)^{\theta(\zeta)} \exp\left(-\zeta\theta(\zeta)(t\land T)G^+(\lambda_0)\right)$$ is a non-negative supermartingale, where we recall the notation $T:=\TPM\land \TK{\lambda_0-\epsilon}$. Since this is the exponential of the process in \eqref{eq: premg estimate}, we have the containment $$ \{\TPM\le t_-, \TK{\lambda_0-\epsilon}\ge \TPM\}\subset \left\{ Z_T\ge \left(\frac{M}{|\mu_0|}\right)^{\alpha(\zeta)\theta(\zeta)}\right\} $$ which has, by optional stopping, probability at most \begin{equation}
	\label{eq: mg est} \PP_{\mu_0, \lambda_0}\left(Z_T\ge \left(\frac{M}{|\mu_0|}\right)^{\alpha(\zeta)\theta(\zeta)}\right) \le \left(\frac{|\mu_0|}{M}\right)^{\alpha(\zeta)\theta(\zeta)}\EE_{\mu_0, \lambda_0}[Z_0]=\left(\frac{|\mu_0|}{M}\right)^{\alpha(\zeta)\theta(\zeta)}.
\end{equation} The conclusion \eqref{eq: upper tail} follows by optimising in $\zeta$ to find the exponent $$\vartheta:=\sup_{\zeta\in (1, (1-\delta)^{-1/2})} \alpha(\zeta)\theta(\zeta).$$ The bound \eqref{eq: lower tail}, possibly with a different exponent, follows by applying the same reasoning to supermartingales of the form $$ Z_t:=\left(\frac{|\mu(t\land T)|}{|\mu_0|}\right)^{-\theta}\exp\left(\zeta \theta(t\land T)G^-(\lambda_0)\right).$$ \end{proof}

\paragraph{Acknowledgements} The first author acknowledges support by the Max Planck Society through the Research Group ``Stochastic Analysis in the Sciences (SAiS)'' and the DFG CRC/TRR 388 ``Rough Analysis, Stochastic Dynamics and Related Fields'', Project A11. The second author is supported by the Royal Commission for the Exhibition of 1851 and acknowledges support by the European Union (ERC, FluCo, grant agreement No. 101088488), as well as the hospitality of the group ``Pattern Formation, Energy Landscapes and Scaling Laws'' of the Max Planck Institute for Mathematics in the Sciences. Views and opinions expressed are however those of the author(s) only and do not necessarily reflect those of the European Union or of the European Research Council. Neither the European Union nor the granting authority can be held responsible for them.

  \bibliographystyle{plain}
\bibliography{literature}

\end{document}